\documentclass[10pt]{article}

\usepackage{graphicx}
\usepackage{amstext}

\usepackage{amsmath, amsthm, amssymb}
\usepackage{appendix}
\usepackage[UKenglish]{babel}
\usepackage{proof}
\usepackage{bussproofs}
\usepackage{fancyhdr}
\usepackage[cmtip,arrow]{xy}
\usepackage{pb-diagram,pb-xy}
\usepackage{mathrsfs}
\usepackage{algorithm}
\usepackage[noend]{algpseudocode}
\usepackage{color}
\usepackage{cite}
\usepackage{courier} 
\usepackage{latexsym}
\usepackage{amsfonts}
\usepackage{wasysym}
\usepackage{microtype}
\usepackage{verbatim}
\usepackage{setspace}
\usepackage{hyperref}
\usepackage[affil-it]{authblk} 

\usepackage{tikz} 

\numberwithin{equation}{section}
\numberwithin{figure}{section}

\newtheorem{thm}{Theorem}[section]
\newtheorem{cor}[thm]{Corollary}
\newtheorem{lem}[thm]{Lemma}
\theoremstyle{remark}

\theoremstyle{definition}

\newtheorem{eg}{Example}

\newcommand{\pair}[1]{\left({#1}\right)}
\newcommand{\ang}[1]{\left\langle{#1}\right\rangle}
\newcommand{\set}[1]{\left\{{#1}\right\}}

\newcommand{\es}{\varnothing}

\newcommand{\infl}{\varepsilon}
\newcommand{\pow}{\mathcal{P}}
\newcommand{\nat}{\mathbb{N}}

\newcommand{\x}{\times}
\newcommand{\sqbra}[1]{\left[{#1}\right]}

\newcommand{\ncsnakes}{\not\kern-0.01cm\mid\!\sim^{\bf{C}}}

\newcommand{\npsnakes}{\not\kern-0.01cm\mid\!\sim^{\bf{P}}}

\newcommand{\nrsnakes}{\not\kern-0.01cm\mid\!\sim^{\bf{R}}}

\newcommand{\lang}{\mathcal{L}}

\newcommand{\LSent}{\mathcal{SL}}

\newcommand{\alg}{\mathcal{A}}

\newcommand{\LForm}{\mathcal{FL}}

\newcommand{\relsymb}{\mathcal{R}}

\newcommand{\ext}{\mathcal{E}}

\newcommand{\powfin}{\pow_\text{fin}}
\newcommand{\subarg}{\subseteq_\text{arg}}

\newcommand{\ordneq}{\triangleleft}
\newcommand{\attk}{\rightharpoonup}
\newcommand{\defeat}{\hookrightarrow}

\newcommand{\orddeneq}{\triangleleft_{DEli}}

\newcommand{\abs}[1]{\left\lvert{#1}\right\rvert}

\begin{document}

\title{Prioritised Default Logic as Argumentation with Partial Order Default Priorities}
\author{Anthony P. Young, Sanjay Modgil, Odinaldo Rodrigues}
\affil{\normalfont\small Department of Informatics, King's College London, Strand, London, U.K.\\
\normalfont\small\texttt{\{\href{mailto:peter.young@kcl.ac.uk}{peter.young},\href{mailto:sanjay.modgil@kcl.ac.uk}{sanjay.modgil},\href{mailto:odinaldo.rodrigues@kcl.ac.uk}{odinaldo.rodrigues}\}@kcl.ac.uk}}
\vspace{-0.5cm}
\date{\today} 
\maketitle

\vspace{-0.8cm}

\begin{abstract}
\noindent We express Brewka's prioritised default logic (PDL) as argumentation using ASPIC$^+$. By representing PDL as argumentation and designing an argument preference relation that takes the argument structure into account, we prove that the conclusions of the justified arguments correspond to the PDL extensions. We will first assume that the default priority is total, and then generalise to the case where it is a partial order. This provides a characterisation of non-monotonic inference in PDL as an exchange of argument and counter-argument, providing a basis for distributed non-monotonic reasoning in the form of dialogue.\footnote{The results of Section \ref{sec:ASPIC+_to_PDL} first appeared in the preprint \cite{APY_ArXiV:2016} and have been published in the conference proceedings of AAMAS2016 \cite{AAMAS2016}. This paper gives the full proofs of these results.}
\end{abstract}

\tableofcontents


\section{Introduction}\label{sec:intro}



Dung's \textit{abstract argumentation theory} \cite{Dung:95} has become established as a means for unifying various nonmonotonic logics (NMLs) \cite{Bochman:07,Reiter:87,SEP:LogicAI}, where the inferences of a given NML can be interpreted as conclusions of justified arguments. Abstract argumentation defines ``justified arguments'' by making use of principles familiar in everyday reasoning and debate. This renders the process of inference in the NML transparent and amenable to human inspection and participation, and serves as a basis for distributed reasoning and dialogue.

More precisely, relating NMLs and argumentation is to endow the NML with \textit{argumentation semantics}. This has already been done for default logic \cite{Dung:95}, logic programming \cite{Dung:95}, defeasible logic \cite{Governatori:04} and preferred subtheories \cite{sanjay:13}. This allows the application of argument game proof theories \cite{Sanjay:09} to the process of inference in these NMLs, and the generalisation of these dialectical proof theories to distributed reasoning amongst computational agents, where agents can engage in argumentation-based dialogues
\cite{Added_Value,Sanjay:08,Atkinson:05}.

Abstract argumentation has been upgraded to \textit{structured argumentation theory} \cite{Besnard:14}, one example of which is the \textit{ASPIC$^+$ framework for structured argumentation} \cite{sanjay:13}. In ASPIC$^+$, arguments are constructed from premises and deductive or defeasible rules of inference. The conclusions of arguments can contradict each other and hence arguments can attack each other. A preference relation over the arguments can be used to determine which attacks succeed as defeats. The arguments and defeats instantiate an abstract argumentation framework, where the justified arguments are determined using Dung's method. The conclusions of the justified arguments are then identified with the nonmonotonic inferences from the underlying premises and rules of inference. The advantages of ASPIC$^+$ are that the framework provides a systematic and general method of endowing non-monotonic logics with argumentation semantics, and identifies sufficient conditions on the underlying logic and preference relations that guarantee the satisfaction of various normatively rational desiderata \cite{Caminada:07}.


This paper endows Brewka's \emph{prioritised default logic} (PDL) \cite{Brewka:94} with argumentation semantics. PDL is an important NML because it upgrades default logic (DL) \cite{Reiter:80} with an explicit priority relation over defaults, so that, for example, one can account for recent information taking priority over information in the distant past. PDL has also been used to represent the (possibly conflicting) beliefs, obligations, intentions and desires (BOID) of agents, and model how these different categories of mental attitudes override each other in order to generate goals and actions that attain those goals \cite{BOID:02}.

We prove a correspondence between inferences in PDL and the conclusions of the justified arguments defined by the argumentation semantics. We realise these contributions by appropriately representing PDL in ASPIC$^+$. The main challenges involve understanding how priorities over defaults in PDL can be represented as an ASPIC$^+$ argument preference relation, and then applying the properties of this preference relation to prove that the extensions of PDL correspond to the conclusions of justified arguments.

This paper has five sections. In Section \ref{sec:rev}, we review ASPIC$^+$, abstract argumentation, and PDL. In Section \ref{sec:ASPIC+_to_PDL} we present an instantiation of ASPIC$^+$ to PDL when the default priority is total. The key results are the design of an appropriate argument preference relation (Section \ref{sec:pref}), and showing that this argument preference relation guarantees that the conclusions of the justified arguments correspond exactly to the PDL extensions by the \textit{representation theorem} (Section \ref{sec:rep_thm}). We then investigate some properties and directly prove that the normative rationality postulates of \cite{Caminada:07} are satisfied (Section \ref{sec:normative_rationality_of_PDL_inst}).\footnote{But in this case we are not leveraging the properties of ASPIC$^+$ to achieve this. We will discuss this point in Section \ref{sec:discussion_conclusions}.}

In Section \ref{sec:lift_total_assumption} we lift the assumption that the priority order on the defaults is total. Following the pattern of the previous section we generalise the argument preference (Section \ref{sec:partial_order_pref_part2}) to accommodate for partial order default priorities. We then prove a generalised representation theorem (Sections \ref{sec:rep_thm_nonlinear}) and prove a partial result concerning the satisfaction of the rationality postulates of \cite{Caminada:07} (Section \ref{sec:rationality_nonlinear}). We conclude in Section \ref{sec:discussion_conclusions} with suggestions for future work.

\section{Background}\label{sec:rev}

\subsection{Notation Used in this Paper}\label{sec:notation}

\noindent In this paper: ``$:=$'' is read ``is defined as''. WLOG stands for ``without loss of generality''. $\nat$ denotes the set  of natural numbers. We denote set difference with $-$. For two sets $A,\:B$, $A\ominus B:=\pair{A-B}\cup\pair{B-A}$ denotes their symmetric difference. If $f:X\to Y$ is a function and $A\subseteq X$, $f(A)\subseteq Y$ is the image set of $A$ in $Y$ under $f$. For a set $X$ its power set is $\pow\pair{X}$ and its finite power set (set of all finite subsets) is $\powfin(X)$. $X\subseteq_\text{fin}Y$ iff $X$ is a finite subset of $Y$, therefore $X\in\powfin(Y)\Leftrightarrow X\subseteq_\text{fin}Y$. Undefined quantities are denoted by $*$, for example $1/0=*$ in the real numbers. Order isomorphism is denoted by $\cong$.

If $\ang{P,\:\lesssim}$ is a preordered set then the strict version of the preorder is $a<b\Leftrightarrow\sqbra{a\lesssim b,\:b\not\lesssim a}$, which is also a strict partial order. If $<$ is a strict partial order on $P$ and $U\subseteq P$, then we define the set $\max_< U:=\set{x\in U\:\vline\:\pair{\forall y\in U}x\not<y}\subseteq U$, i.e. the set of all $<$-maximal elements of $U$. We define the set $\min_< U$ analogously. For a set $X$ we define the set of possible strict partial orders on $X$ to be $PO(X):=\set{<\:\subseteq\:X^2\:\vline\:\text{$<$ is a strict partial order}}$. Similarly, the set of all possible strict total orders on $X$ is $TO(X):=\set{<\:\subseteq\:X^2\:\vline\:\text{$<$ is a strict total order}}\subset PO(X)$. We will use the terms ``total (order)'' and ``linear (order)'' interchangeably. We will also call totally ordered sets either ``tosets'' or ``chains''.

\subsection{The ASPIC\texorpdfstring{$^+$}{+} Framework}\label{sec:rev_ASPIC+}


\noindent Abstract argumentation abstracts from the internal logical structure of arguments, the nature of defeats and how they are determined by preferences, and consideration of the conclusions of the arguments \cite{Dung:95}. However, these features are referenced when studying whether any given logical instantiation of a framework yields complete extensions that satisfy the rationality postulates of \cite{Caminada:07}. ASPIC$^+$ \cite{sanjay:13} provides a structured account of abstract argumentation, allowing one to reference the above features, while at the same time accommodating a wide range of instantiating logics and preference relations in a principled manner. ASPIC$^+$ then identifies conditions under which complete extensions defined by the arguments, attacks and preferences, satisfy the rationality postulates of \cite{Caminada:07}.


In ASPIC$^+$, the tuple $\ang{\lang,\:-,\:\relsymb_s,\:\relsymb_d,\:n}$ is an \emph{argumentation system}, where $\lang$ is a logical language and $-:\lang\to\pow(\lang)$ is the \emph{contrary function} $\theta\mapsto\overline{\theta}$ where $\overline{\theta}$ is the set of wffs that are inconsistent with $\theta$. Let $\theta_1,\dots,\theta_m, \phi\in\lang$ be wffs for $m\in\nat$, $\relsymb_s$ is the \emph{set of strict inference rules} of the form $(\theta_1,\:\ldots,\:\theta_m\to\phi)$, denoting that if $\theta_1,\:\ldots,\:\theta_m$ are true then $\phi$ is also true, and $\relsymb_d$ is the \emph{set of defeasible inference rules} of the form $(\theta_1,\:\ldots,\:\theta_m\Rightarrow\phi)$, denoting that if $\theta_1,\:\ldots,\:\theta_m$ are true then $\phi$ is tentatively true. Note $\relsymb_s\cap\relsymb_d=\es$. For a strict or defeasible rule $r=\pair{\theta_1,\ldots\theta_m\to/\Rightarrow\phi}$, we define $Ante(r):=\set{\theta_1,\ldots,\theta_m}\subseteq_\text{fin}\lang$,\footnote{Note it is possible to have $m=0$ and hence $Ante(r)=\es$.} and $Cons(r):=\phi\in\lang$. Finally, $n:\relsymb_d\to\lang$ is a \emph{partial} function that assigns a \emph{name} to \emph{some} of the defeasible rules. For any $S\subseteq\lang$ we define the set $Cl_{\relsymb_s}(S)\subseteq\lang$ to be the smallest superset of $S$ that also contains $Cons(r)$ for all $r\in\relsymb_s$ such that $Ante(r)\subseteq Cl_{\relsymb_s}(S)$. We call $Cl_{\relsymb_s}$ the \textit{closure under strict rules operator}.

In ASPIC$^+$, a \emph{knowledge base} is a set $\mathcal{K}:=\mathcal{K}_n\cup\mathcal{K}_p\subseteq\lang$ where $\mathcal{K}_n$ is the set of \emph{axioms} and $\mathcal{K}_p$ is the set of \emph{ordinary premises}. Note that $\mathcal{K}_n\cap\mathcal{K}_p=\es$. Given an argumentation system and $\mathcal{K}$, \emph{arguments} are defined inductively:
\begin{enumerate}
\item(Base) $[\theta]$ is a \emph{singleton argument} with $\theta\in\mathcal{K}$, conclusion $Conc([\theta]):=\theta$, premise set $Prem\pair{\sqbra{\theta}}:=\set{\theta}\subseteq\mathcal{K}$, top rule $TopRule([\theta]):=*$ and set of subarguments $Sub\pair{[\theta]}:=\set{[\theta]}$.
\item(Inductive) Let $A_1,\:\ldots,\:A_n$ be arguments with respective conclusions $Conc(A_1),\:\ldots,\:Conc(A_n)$ and premise sets $Prem(A_1),\:\ldots,\:Prem(A_n)$. If there is a rule $r:=(Conc\pair{A_1},\:\ldots,\:Conc\pair{A_n}\to / \Rightarrow\phi)\in\relsymb$, then $B:=[A_1,\:\ldots,\:A_n\to/\Rightarrow\phi]$ is also an argument with $Conc(B)=\phi$, premises $Prem(B):=\bigcup_{i=1}^n Prem(A_i)$, $TopRule(B)=r\in\relsymb$ and set of subarguments $Sub(B):=\set{B}\cup\bigcup_{i=1}^n Sub(A_i)$.
\end{enumerate}

\noindent Let $\alg$ be the (unique) set of all arguments freely constructed following the above rules. It is clear that arguments are finite objects in that each argument has finitely many premises, and take finitely many rules to reach its conclusion. We define the \textit{conclusion map} $Conc:\alg\to\lang:A\mapsto Conc(A)$. We can generalise this to arbitrary \textit{sets} of arguments (abuse of notation):
\begin{align}\label{eq:conc_for_sets}
Conc:\pow\pair{\alg}&\to\pow\pair{\lang}:S\mapsto Conc(S):=\bigcup_{A\in S}Conc(A).
\end{align}

Two strict or defeasible rules are \textit{equal} iff they have the same antecedent sets, consequents and name syntactically in the underlying $\lang$. Two arguments are \emph{equal} iff they are constructed identically as described above. More precisely, we can define equality of arguments inductively. The base case would be two singleton arguments $[\theta],\:[\phi]$ are equal iff $\theta$ and $\phi$ are \textit{syntactically} the same formulae. Given $n$ arguments $A_1,\:\ldots,\:A_n$ and two equal rules $r_1$ and $r_2$ (either both strict or both defeasible) with antecedent $\set{Conc\pair{A_i}}_{i=1}^n$, such that $B_1$ is the rule $r_1$ appended to the $A_i$'s, and $B_2$ is the rule $r_2$ appended to the $A_i$'s, then $B_1$ and $B_2$ are equal arguments.


We say $A$ is a \textit{subargument} of $B$ iff $A\in Sub(B)$ and we write $A\subarg B$. We say $A$ is a \textit{proper subargument} of $B$ iff $A\in Sub(B)-\set{B}$ and we write $A\subset_\text{arg}B$. It can be shown that $\subarg$ is a preorder on $Sub(B)$. A set of arguments is \textit{subargument closed} iff it is $\subarg$-down closed. Clearly, for every defeasible rule $r$ in an argument $A$, there is a subargument of $A$ with $r$ as its top rule, by the inductive construction of arguments.

An argument $A\in\alg$ is \emph{firm} iff $Prem(A)\subseteq\mathcal{K}_n$. Further, $SR(A)\subseteq\relsymb_s$ is the set of strict rules applied in constructing $A$, and $DR(A)\subseteq\relsymb_d$ is the set of defeasible rules applied in constructing $A$. We also define $Prem_n(A):=Prem(A)\cap\mathcal{K}_n$ and $Prem_p(A):=Prem(A)\cap\mathcal{K}_p$. An argument $A$ is \emph{strict} iff $DR(A)=\es$, else $A$ is \emph{defeasible}. We can generalise $DR\pair{\:\cdot\:}$ to sets as well just like Equation \ref{eq:conc_for_sets} for $Conc\pair{\:\cdot\:}$.

Given $R\subseteq\relsymb_d$, we introduce \emph{the set of all arguments freely constructed with defeasible rules restricted to those in $R$} as the set $Args(R)\subseteq\alg$, which are all arguments with premises in $\mathcal{K}$, strict rules in $\relsymb_s$ and defeasible rules in $R$. Formally, $Args(R)$ is defined inductively just as how arguments are constructed, but with the choice of defeasible rules restricted to those in $R$. It is easy to show that this definition is equivalent to
\begin{align}\label{eq:def_Args(R)}
A\in Args(R)\Leftrightarrow DR(A)\subseteq R.
\end{align}
Clearly, $Args(\relsymb_d)=\alg$. Given $R$, $Args(R)$ exists and is unique. 

Let $S\subseteq_\text{fin}\alg$. The \textit{set of all strict extensions of $S$} is the set $StExt\pair{S}$ where
\begin{align}
A\in StExt\pair{S}\Leftrightarrow &DR(A)=DR(S),\:SR(A)\supseteq SR(S),\nonumber\\
&Prem_p(A)=Prem_p(S),\:Prem_n(A)\supseteq Prem_n(S).\nonumber
\end{align}
A set $S\subseteq\alg$ is \textit{closed under strict extensions} iff for all $T\subseteq_\text{fin}S$, $StExt\pair{T}\subseteq S$.

\begin{lem}\label{lem:Args(R)_props}
The set $Args(R)$, for any $R\subseteq\relsymb_d$, is closed under strict extensions and subarguments.
\end{lem}
\begin{proof}
If $A\in Args(R)$ and $B\subseteq A$, then $DR(B)\subseteq DR(A)\subseteq R$ so $DR(B)\subseteq R$ and hence $B\in Args(R)$, therefore $Args(R)$ is subargument closed. Now let $T\subseteq_\text{fin}Args(R)$, so for all $B\in T$, $DR(B)\subseteq R$, therefore $DR(T):=\bigcup_{B\in T}DR(B)\subseteq R$. Let $A\in StExt\pair{T}$, then $DR(A)=DR(T)\subseteq R$ and hence $A\in Args\pair{R}$. Therefore, $StExt\pair{T}\subseteq Args\pair{R}$, therefore $Args(R)$ is closed under strict extensions.
\end{proof}


An argument $A$ \emph{attacks} another argument $B$, denoted as $A\attk B$, iff at least one of the following hold, where:
\begin{enumerate}
\item  $A$ is said to \emph{undermine} attack $B$ on the (singleton) subargument $B'$ = $[\phi]$ iff there is some $\phi\in Prem_p(B)$ such that $Conc(A)\in\overline{\phi}$.
\item $A$ is said to \emph{rebut} attack $B$ on the subargument $B'$ iff there is some $B'\subarg B$ such that $r:= TopRule\pair{B'}\in\relsymb_d$, $\phi:=Cons(r)$ and $Conc(A)\in\overline{\phi}$.
\item $A$ is said to \emph{undercut} attack $B$ on the subargument $B'$ iff there is some $B'\subarg B$ such that $r:=TopRule(B')\in\relsymb_d$ and $Conc(A)\in\overline{n(r)}$.
\end{enumerate}

\noindent See \cite[Section 2]{sanjay:13} for a further discussion of why attacks are distinguished in this way. We abuse notation to define the \emph{attack relation} as $\attk\:\subseteq\alg^2$ such that $(A,\:B)\:\in\:\attk\:\Leftrightarrow\: A\attk B$. Notice that by the transitivity of $\subarg$, if $A\attk B$ and $B\subarg C$, then $A\attk C$.

A preference relation over arguments is then used to determine which attacks  succeed as defeats. We denote the preference $\precsim\:\subseteq\:\alg^2$ (not necessarily a preorder for now) such that $A\precsim B\Leftrightarrow\:B$ is at least as preferred as $A$. Strict preference and equivalence are, respectively, $A\prec B\Leftrightarrow\sqbra{A\precsim B,\:B\not\precsim A}$ and $A\approx B\Leftrightarrow\sqbra{A\precsim B,\:B\precsim A}$. We define a \emph{defeat} as
\begin{align}\label{eq:ASPIC+_general_defeat}
A\defeat B\Leftrightarrow\pair{\exists B'\subarg B}\sqbra{A\attk B',\:A\not\prec B'}.
\end{align}
That is to say, $A$ defeats $B$ (on $B'$) iff  $A$ attacks $B$ on the subargument $B'$, and $B'$ is not 
strictly preferred to $A$. Notice the comparison is made at the subargument $B'$ instead of the whole argument $B$. We abuse notation to define the \emph{defeat relation} as $\defeat\:\subseteq\alg^2$ such that $(A,\:B)\:\in\:\defeat\:\Leftrightarrow\: A\defeat B$. A set of arguments $S\subseteq\alg$ is \textit{conflict-free} (cf) iff $\defeat\:\cap\:S^2=\es$.\footnote{Note that \cite{sanjay:13} studies two different notions of cf sets: one where no two arguments \textit{attack} each other, and the other where no two arguments \textit{defeat} each other. We choose the latter notion of cf as this is more commonplace in argumentation formalisms that distinguish between attacks and defeats, e.g. in \cite{Prakken:10}.} Notice that by the transitivity of $\subarg$ and that the preference comparison is made at the defeated subargument, if $A\defeat B$ and $B\subarg C$, then $A\defeat C$. As relations, $\defeat\:\subseteq\:\attk$.

Preferences between arguments are calculated from the argument structure by comparing arguments at their fallible components, i.e. the ordinary premises and defeasible rules. This is achieved by endowing $\mathcal{K}_p$ and $\relsymb_d$ with preorders $\lesssim_K$ and $\lesssim_D$ respectively, where (e.g.) $r_1\lesssim_D r_2$ iff $r_2$ is \textit{just as preferred or more preferred} than $r_1$ (and analogously for $\lesssim_K$). These preorders are then aggregated to a set-comparison relation $\trianglelefteq$ between the sets of premises and / or defeasible rules of the arguments, and then finally to $\precsim\:\subseteq\:\alg^2$, following the method in \cite[Section 5]{sanjay:13}.\footnote{Note there are many other ways to lift a preference $<$ on a set of objects $X$ to compare subsets of $X$ in various ways that are ``compatible'' with $<$ \cite{Barbera:04}.} We will use a modified version of this lifting, which will be explained in Section \ref{sec:pref}.

Given the preference relation $\precsim$ between arguments, we call the structure $\ang{\alg,\:\attk,\:\precsim}$ an \emph{ASPIC$^+$ SAF} (structured argumentation framework), or \emph{attack graph}. Its corresponding \emph{defeat graph} is $\ang{\alg,\:\defeat}$, where $\defeat$ is defined in terms of $\attk$ and $\precsim$ as in Equation \ref{eq:ASPIC+_general_defeat}.


Given $\ang{\alg,\:\defeat}$ one can then evaluate the extensions under Dung's abstract argumentation semantics, and thus identify the inferences defined by argumentation as the conclusions of the justified arguments. We now recap the key definitions of \cite{Dung:95}. An \emph{argumentation framework} is a directed graph $\ang{\alg,\:\defeat}$, where $\alg$ is the set of arguments and $\defeat\:\subseteq\:\alg^2$ is the \emph{defeat relation}, such that $A\defeat B$ means $A$ is a (successful) counterargument against $B$. The argumentation frameworks we consider are defeat graphs, but this is a general definition. 

Let $S\subseteq\alg$ and $A,\:B\in\alg$. $S$ \emph{defeats} $B$ iff $\pair{\exists A\in S}A\defeat B$. $S$ is \emph{conflict-free} (cf) iff $\defeat\cap S^2=\es$. $S$ \emph{defends} $A$ iff $\pair{\forall B\in\alg}[B\defeat A\Rightarrow S$  defeats $B]$. The \textit{characteristic function} is $\chi:\pow\pair{\alg}\to\pow\pair{\alg}$, such that $\chi(S):=\set{A\in\alg\:\vline\:S\text{ defends }A}\subseteq\alg$. $S$ is an \emph{admissible extension} iff $S$ is cf and $S\subseteq \chi(S)$. An admissible extension $S$ is: a \emph{complete extension} iff $S=\chi(S)$; a \emph{preferred extension} iff $S$ is a $\subseteq$-maximal complete extension; the \emph{grounded extension} iff $S$ is the $\subseteq$-least complete extension; a \emph{stable extension} iff $S$ is complete and defeats all arguments $B\in\alg-S$.

Let $\mathcal{S}:=\{complete,preferred,grounded,stable\}$ be the set of \emph{Dung semantics}. An argument $A\in\alg$ is \emph{sceptically (credulously) justified} under the semantics $s\in\mathcal{S}$ iff $A$ belongs to all (at least one) of the $s$-extensions of $\ang{\alg,\defeat}$.

Instantiations of ASPIC$^+$ should satisfy some properties to ensure they are rational \cite{Caminada:07}. Given an instantiation let $\ang{\alg,\:\attk,\:\precsim}$ be its ASPIC$^+$ attack graph with corresponding defeat graph $\ang{\alg,\defeat}$. Let $\ext$ be any complete extension. The \textit{Caminada-Amgoud rationality postulates} state:
\begin{enumerate}
\item (Subargument closure) $\ext$ is subargument closed.
\item (Closure under strict rules) $\ext$ satisfies $Conc(\ext)=Cl_{\relsymb_s}\sqbra{Conc(\ext)}$, where $Conc\pair{\ext}$ is defined in Equation \ref{eq:conc_for_sets}.
\item (Consistency) $Conc(\ext)$ is  consistent.\footnote{Notice by properties 2 and 3 above $Cl_{\relsymb_s}\pair{Conc\pair{\ext}}$ is consistent. ASPIC$^+$ distinguishes this into direct and indirect consistency given that $\relsymb_s$ is in general arbitrary and do not have to be the rules of inference of classical logic. We will not make this distinction because our underlying logic will be first order logic (FOL) (Section \ref{sec:instantiation_choice_of_variables}). Further, consistency in the abstract logic of ASPIC$^+$ is expressed in terms of the contrary function, but since our contrary function will just be classical negation, we can take the usual meaning of consistency in FOL.}
\end{enumerate}
An ASPIC$^+$ instantiation is \textit{normatively rational} iff it satisfies these rationality postulates. These postulates may be proved directly given an instantiation. ASPIC$^+$ also identifies sufficient conditions for an instantiation to satisfy these postulates \cite[Section 4]{sanjay:13}, which we will discuss in Section \ref{sec:discussion_conclusions}.

\subsection{Brewka's Prioritised Default Logic}\label{sec:rev_PDL}

\noindent In this section we recap Brewka's prioritised default logic (PDL) \cite{Brewka:94}. We work in first order logic (FOL) of arbitrary signature where the set of first-order formulae is $\LForm$ and the set of closed first order formulae\footnote{i.e. first order formulae without free variables} a.k.a. \textit{sentences} is $\LSent\subseteq\LForm$, with the usual quantifiers and connectives. Entailment is denoted by $\models$. Logical equivalence of formulae is denoted by $\equiv$. Given $S\subseteq\LForm$, the \emph{deductive closure of $S$} is $Th(S)$, and given $\theta\in\LForm$, the \emph{addition operator} $+:\pow(\LForm)\x\LForm\to\pow\pair{\LForm}$ is defined as $S+\theta:=Th(S\cup\set{\theta})$.

A \emph{normal default} is an expression $\frac{\theta:\phi}{\phi}$ where $\theta,\:\phi\in\LForm$ and read ``if $\theta$ is the case and $\phi$ is consistent with what we know, then jump to the conclusion $\phi$ even if it does not deductively follow''. In this case we call $\theta$ the \emph{antecedent} and $\phi$ the \emph{consequent}. A normal default $\frac{\theta:\phi}{\phi}$ is \emph{closed} iff $\theta,\:\phi\in\LSent$. We will assume all defaults are closed and normal unless stated otherwise. Given $S\subseteq\LSent$, a default is \emph{active (in $S$)} iff $\sqbra{\theta\in S,\:\phi\notin S,\:\neg\phi\notin S}$. 

A \emph{finite prioritised default theory} (PDT) is a structure $T:=\ang{D,\:W,\:<}$, where \textit{the set of facts} $W\subseteq\LSent$ is not necessarily finite and $\ang{D,\:<}$ is a \emph{finite} strict partially ordered set of defaults that nonmonotonically extend $W$. The priority relation is such that $d'< d\Leftrightarrow d$ is \emph{more\footnote{\label{fn:dual_priority_PDL} We have defined the order dually to \cite{Brewka:94} so as to comply with orderings over the ASPIC$^+$ defeasible inference rules. This goes \textit{against} the tradition in NML where the \textit{smaller} item in $<$ is the \textit{more} preferred one.} prioritised} than $d'$. All PDTs in this paper are finite.

The inferences of a PDT $T=\ang{D,W,<}$ are defined by its extensions. Let $<^+\supseteq <$ be a linearisation of $<$. A \emph{prioritised default extension (with respect to $<^+$)} (PDE) is a set $E:=\bigcup_{i\in\nat}E_i\subseteq\LSent$ built inductively as:
\newpage
\begin{align}
E_0&:=Th(W)\text{ and }\label{eq:ext_base}\\
E_{i+1}&:=
\begin{cases}
E_i+\phi, &\text{if property 1}\\
E_i,&\text{else}
\end{cases}\label{eq:ext_ind}
\end{align}
where ``property 1'' abbreviates ``$\phi$ is the consequent of the $<^+$-\textit{greatest}\footnote{See Footnote \ref{fn:dual_priority_PDL}.} default $d$ active in $E_i$''. Intuitively, one first generates all classical consequences from the facts $W$, and then iteratively adds the nonmonotonic consequences from the highest priority default to the lowest. Notice if $W$ is inconsistent then $E_0=E=\LForm$. For this paper we will assume $W$ is always consistent.

For finite $D$ it can be shown that the ascending chain $E_i\subseteq E_{i+1}$ stabilises at some finite $i\in\nat$ and that $E$ is consistent provided that $W$ is consistent. $E$ does not have to be unique because there are many distinct linearisations of $<$. We say $T$ \emph{sceptically infers} $\theta\in\LSent$ iff $\theta\in E$ for \textit{all} extensions $E$ of $T$.

A PDT $T$ for which $<$ is a strict total order is a \emph{linearised PDT} (LPDT). If $<$ is total then there is only one way to apply the defaults in $D$ by Equation \ref{eq:ext_ind}, hence the extension is unique. We will use the notation $<^+$ to emphasise that the priority is total, and the notation $T^+$ to denote an arbitrary LPDT.

For the rest of this paper, if we declare $T$ to be a PDT, we mean $T=\ang{D,W,<}$ where each component is defined above, and we make no further assumptions on each component. If we declare $T^+$ to be an LPDT, we mean $T^+=\ang{D,W,<^+}$ where $<^+$ is a strict total order on $D$.



\section{From ASPIC\texorpdfstring{$^+$}{+} to PDL}\label{sec:ASPIC+_to_PDL}

\subsection{Representing PDL in ASPIC\texorpdfstring{$^+$}{+}}\label{sec:instantiation_choice_of_variables}

\noindent We now instantiate ASPIC$^+$ to PDL. Let $T^+:=\ang{D,\:W,\:<^+}$ be an LPDT.\footnote{We will lift this assumption of a total order priority in Section \ref{sec:lift_total_assumption}.}
\begin{enumerate}
\item Our arguments are expressed in FOL, so our set of wffs $\lang$ is $\LForm$.
\item The contrary function $-:\LForm\to\pow\pair{\LForm}$ syntactically defines conflict in terms of classical negation. For all $\theta\in\LForm$, $\overline{\theta}=\set{\neg\theta}$ unless $\theta$ has the syntactic form $\neg\phi$ for some $\phi\in\LForm$, then $\overline{\theta}=\set{\phi}$. As $\overline{\theta}$ is singleton, we will abuse notation and write $\overline{\theta}$ to refer to its element.
\item The set of strict rules $\relsymb_s$ characterises inference in FOL. Notice $\relsymb_s$ is \textit{closed under transposition}, i.e. for all $1\leq i\leq n\in\nat^+$,
\begin{align}
&\pair{\theta_1,\theta_2,\ldots,\theta_{i-1},\theta_i,\theta_{n+1},\ldots,\theta_n\to\phi}\in\relsymb_s\nonumber\\
\Rightarrow&\pair{\theta_1,\theta_2,\ldots,\theta_{i-1},-\phi,\theta_{n+1},\ldots,\theta_n\to-\theta_i}\in\relsymb_s.\nonumber
\end{align}
We leave the proof theory implicit. $Cl_{\relsymb_s}$ instantiates to deductive closure.
\item The set of defeasible rules $\relsymb_d$ is defined as:
\begin{align*}
\relsymb_d:=\set{(\theta\Rightarrow\phi)\:\vline\:\frac{\theta:\phi}{\phi}\in D},
\end{align*}
with the naming function $n\equiv*$. Clearly, there is a bijection $f$ where
\begin{align}\label{eq:bij_defaults_def_rules}
f:D\to\relsymb_d:\frac{\theta:\phi}{\phi}\mapsto f\pair{\frac{\theta:\phi}{\phi}}:=\pair{\theta\Rightarrow\phi}
\end{align}
and we will define the \emph{strict version of the} preorder $\leq_D$ over $\relsymb_d$ as\footnote{From Footnote \ref{fn:dual_priority_PDL}, we do not need to define $<_D$ as the order-theoretic dual to $<^+$, avoiding potential confusion as to which item is more preferred.}
\begin{align}\label{eq:def_rules_pref_order}
(\theta\Rightarrow\phi)<_D(\theta'\Rightarrow\phi')\Leftrightarrow\frac{\theta:\phi}{\phi}<^+\frac{\theta':\phi'}{\phi'}.
\end{align}
We can see that the strict toset $\ang{\relsymb_d,\:<_D}$ is order isomorphic to $\ang{D,\:<^+}$.
\item The set of axiom premises is $\mathcal{K}_n=W$, because we take $W$ to be the set of facts. Furthermore, $\mathcal{K}_p=\es$.
\end{enumerate}

\noindent The set $\alg$ of ASPIC$^+$ arguments are defined as in Section \ref{sec:rev_ASPIC+}.\footnote{As $\relsymb_s$ is a countably infinite set, $\alg$ is also a countably infinite set.} All arguments are firm because $\mathcal{K}_p=\es$, and so there are no undermining attacks. As $n$ is undefined, no attack can be an undercut. Therefore, we only have rebut attacks, 
\begin{align}\label{eq:attack}
A\attk B\Leftrightarrow\pair{\exists B',\:B''\subarg B}\:B'=\sqbra{B''\Rightarrow\overline{Conc(A)}}.
\end{align}
Defeats are defined as in Equation \ref{eq:ASPIC+_general_defeat}. In the next section, we will define the argument preference $\precsim$, based on the strict total order $<_D$ over $\relsymb_d$.

\subsection{A Suitable Argument Preference Relation}\label{sec:pref}


We wish to define a suitable argument preference relation such that the conclusion set of the stable extension defined by $\defeat$ corresponds to the extension of the underlying PDT.\footnote{In Section \ref{sec:normative_rationality_of_PDL_inst}, we will show that for the resulting defeat graphs there is only one extension in that is stable, grounded and preferred.} The first place to look for such a relation is in the existing relations of ASPIC$^+$ \cite[Definition 19]{sanjay:13}. However, simple counterexamples can be devised to show the inferences of the PDT and its argumentation counterpart do not correspond.

The difference between PDL and ASPIC$^+$ is in how blocked defaults are treated. In PDL, blocked defaults are simply excluded from the extension. In ASPIC$^+$, it is possible to construct arguments with defeasible rules that correspond to blocked defaults. If $<^+\:\cong\:<_D$ such that $<^+$ is arbitrary, there is no guarantee that the blocked defaults will be positioned in the chain $<^+$ such that arguments with blocked defaults are always defeated by arguments with only non-blocked defaults.\footnote{We will see this explicitly in Example \ref{eg:not_disj_eli_WLP} later.} To ensure that arguments with blocked defaults are defeated and hence the conclusions of the justified arguments form the extension of the PDT, we need to rearrange the rules in $\relsymb_d$ to take into account the structure of arguments. ASPIC$^+$ does allow for explicit reference to argument structure, i.e. we can tell which defeasible rules preceed which within an argument.

Rearranging $<_D$ to take argument structure into account captures the PDL meaning of ``active'' default, because defaults are added to $E_i$ when its prerequisite is inferred. This rearrangement will mean that every defeasible rule $r$ corresponding to a blocked default will be less preferred than the rules which make up arguments that rebut the argument with $r$ as its top rule. We now devise a new ASPIC$^+$ argument preference relation which incorporates the argument structure into the preorder $<_D$.


More formally, given any strict total order $<_D$ on $\relsymb_d$, we first define a transformation $<_D\:\mapsto\:<_{SP}$, where the subscript SP stands for \textit{structure-preference}. This sorts the defeasible rules in a way compatible with both the priority $<_D$ and their logical structure.

The set $\relsymb_d$ is finite because we have assumed that $D$ is finite (Sections \ref{sec:rev_PDL} and \ref{sec:instantiation_choice_of_variables}). Let $1\leq i\leq|\relsymb_d|=:N\in\nat$. We define $a_i\in\relsymb_d$ to be the $<_D$-greatest element of the following set:
\begin{align}\label{eq:SP_ord_def}
\set{r\in\relsymb_d\:\vline\:Ante(r)\subseteq Conc\sqbra{Args\pair{\bigcup_{k=1}^{i-1}\set{a_k}}}}\!-\!\bigcup_{j=1}^{i-1}\set{a_j}.
\end{align}

\noindent The intuition is: $a_1$ is the most preferred rule whose antecedent is inferred by the conclusions of all strict arguments, $a_2$ is the next most preferred rule, whose antecedent is amongst the conclusions of all arguments having \textit{at most} $a_1$ as a defeasible rule. Similarly, $a_3$ is the next most preferred rule, whose antecedent is amongst the conclusions of all arguments having \textit{at most} $a_1$ and $a_2$ as defeasible rules, and so on until all of the rules of $\relsymb_d$ are exhausted. Notice that the second union after the set difference in Equation \ref{eq:SP_ord_def} ensures that once a rule is applied it cannot be applied again. We then define $<_{SP}$ as (notice the dual order)
\begin{align}\label{eq:SP_order}
a_i <_{SP} a_j\Leftrightarrow j < i.
\end{align}
We define the non-strict order to be $a_i\leq_{SP}a_i\Leftrightarrow\sqbra{a_i=a_j\text{ or }a_i<_{SP}a_j}$. This makes sense because $i\mapsto a_i$ is bijective between $\relsymb_d$ and $\set{1,2,3,\ldots,\:N}$. Clearly $<_{SP}$ is a strict total order on $\relsymb_d$. We call this the \textit{structure preference order} on $\relsymb_d$, which exists and is unique given $<_D$. This means the transformation $<_D\:\mapsto\:<_{SP}$ is functional, where $<_D$ is total on $\relsymb_d$.

Now let $<_D$ be \textit{any} strict partial order on $\relsymb_d$. We define the strict set comparison relation on $\powfin\pair{\relsymb_d}$ corresponding to $<_D$. For $\Gamma,\:\Gamma'\subseteq_\text{fin}\relsymb_d$, the relation $\orddeneq$, called \textit{the disjoint elitist order}, is defined as follows:
\begin{align}\label{eq:disj_eli}
\Gamma\orddeneq\Gamma'\Leftrightarrow\pair{\exists x\in\Gamma-\Gamma'}\pair{\forall y\in\Gamma'-\Gamma}x<_{D}y.
\end{align}
The lifting $<_{D}\:\mapsto\:\orddeneq$ is functional. We will focus on the following special case of $\orddeneq$, where instead of $<_D$ we have $<_{SP}$:
\begin{align}\label{eq:SP_set_comparison}
\Gamma\ordneq_{SP}\Gamma'\Leftrightarrow\pair{\exists x\in\Gamma-\Gamma'}\pair{\forall y\in\Gamma'-\Gamma}x<_{SP}y.
\end{align}
The corresponding strict argument preference is, for $A,B\in\alg$,
\begin{align}\label{eq:SP_arg_pref}
A\prec_{SP}B\Leftrightarrow DR(A)\ordneq_{SP} DR(B).
\end{align}
We define the corresponding non-strict preference as
\begin{align}\label{eq:SP_arg_pref_non_strict}
A\precsim_{SP}B\Leftrightarrow\sqbra{DR(A)\ordneq_{SP}DR(B)\text{ or } DR(A)=DR(B)}
\end{align}

We now show that $\precsim_{SP}$ satisfies the following properties.


\begin{lem}\label{lem:larger_args_less_pref}
For all $A,B\in\alg$, $DR(A)\subseteq DR(B)\Rightarrow B\precsim_{SP} A$.
\end{lem}
\begin{proof}
If $DR(B)=DR(A)$ then $B\approx A$, so $B\precsim_{SP} A$. If $DR(A)\subset DR(B)$, then $DR(A)-DR(B)=\es$, which means $B\prec_{SP} A$ is vacuously true from Equation \ref{eq:SP_set_comparison} so $B\precsim_{SP} A$.
\end{proof}

\noindent The following result shows that larger arguments, which potentially can contain more fallible information (i.e. defeasible rules), cannot be more preferred than its (smaller) subarguments.

\begin{cor}\label{cor:larger_args_less_pref}
For all $A,B\in\alg$, if $A\subarg B$ then $B\precsim_{SP}A$.
\end{cor}
\begin{proof}
It can be shown from how ASPIC$^+$ arguments are constructed (Section \ref{sec:rev_ASPIC+}) that $A\subarg B\Rightarrow DR(A)\subseteq DR(B)$, and then invoke Lemma \ref{lem:larger_args_less_pref}.
\end{proof}

\begin{cor}\label{cor:empty_set_is_top}
Strict arguments are $\precsim_{SP}$-maximal.
\end{cor}
\begin{proof}
Let $A\in\alg$ be strict and $B\in\alg$ be arbitrary. Assume for contradiction that $A\prec_{SP}B$. As $DR(A)=\es$, Equations \ref{eq:SP_set_comparison} and \ref{eq:SP_arg_pref} instantiate to: $A\prec_{SP}B\Leftrightarrow\pair{\exists x\in\es}\pair{\forall y\in DR(B)}x<_{SP}y$, which is impossible by the first bounded quantifier. Therefore, if $A$ is strict, then for all $\pair{\forall B\in\alg}A\not\prec_{SP}B$.
\end{proof}

\begin{lem}\label{lem:disj_eli_is_transitive_over_chain}
Let $\ang{P,<}$ be a strict toset, then $\ang{\powfin\pair{P},\orddeneq}$ is also a strict toset, where $\orddeneq$ is defined in Equation \ref{eq:disj_eli}, here with $<$ instead of $<_D$.
\end{lem}
\begin{proof}
We prove $\orddeneq$ is irreflexive, transitive and total on $\powfin\pair{P}$, assuming that $<$ is a strict total order on $P$.\footnote{\label{fn:acyclic} More generally, it can be shown that for \textit{any} strict partial order $<$, the relation $\orddeneq$ from Equation \ref{eq:disj_eli} is acyclic, and hence irreflexive and asymmetric, but not necessarily transitive. If $<$ is a \textit{modular} order \cite[Lemma 3.7]{LM:92}, then $\orddeneq$ is transitive. Further, if $<$ total (recalling that total orders are modular), then $\orddeneq$ is trichotomous, and hence a strict total order.} To show irreflexivity, assume for contradiction that there is some $\Gamma\in\powfin\pair{P}$ such that $\Gamma\orddeneq\Gamma$, which by Equation \ref{eq:disj_eli} is equivalent to a formula whose first bounded quantifier is ``$\pair{\exists x\in\es}$'', which is false, so $\orddeneq$ is irreflexive. To show transitivity, let $n_1,\:\cdots,\:n_7\in\nat$, such that
\begin{align}
&\set{a_1,\:\cdots,\:a_{n_1}}\cup\set{b_1,\:\cdots,\:b_{n_2}}\cup\set{c_1,\:\cdots,\:c_{n_3}}\cup\set{d_1,\:\cdots,\:d_{n_4}}\nonumber\\
\cup&\set{e_1,\:\cdots,\:e_{n_5}}\cup\set{f_1,\:\cdots,\:f_{n_6}}\cup\set{g_1,\:\cdots,\:g_{n_7}}\subseteq P.
\end{align}
All of these elements $a_1,\:\ldots,\:g_{n_7}$ are distinct. If $n_i=0$ then the corresponding set is empty. Let $\Gamma,\:\Gamma',\:\Gamma''\subseteq_\text{fin}P$, where
\begin{align*}
\Gamma&=\set{a_1,\:\cdots,\:a_{n_1}}\cup\set{d_1,\:\cdots,\:d_{n_4}}\cup\set{f_1,\:\cdots,\:f_{n_6}}\cup\set{g_1,\:\cdots,\:g_{n_7}},\\
\Gamma'&=\set{b_1,\:\cdots,\:b_{n_2}}\cup\set{d_1,\:\cdots,\:d_{n_4}}\cup\set{e_1,\:\cdots,\:e_{n_5}}\cup\set{g_1,\:\cdots,\:g_{n_7}}\text{ and}\\
\Gamma''&=\set{c_1,\:\cdots,\:c_{n_3}}\cup\set{e_1,\:\cdots,\:e_{n_5}}\cup\set{f_1,\:\cdots,\:f_{n_6}}\cup\set{g_1,\:\cdots,\:g_{n_7}}.
\end{align*}

\noindent We can picture these sets with the Venn diagram in Figure \ref{figure:Venn}. The solid outer rectangle represents the set $P$. The three finite sets $\Gamma,\:\Gamma',\:\Gamma''$ are the three rectangles within. Each overlapping region has exactly the elements indicated and nothing more. This configuration exhausts all possibilities for $\Gamma,\:\Gamma'$ and $\Gamma''$.

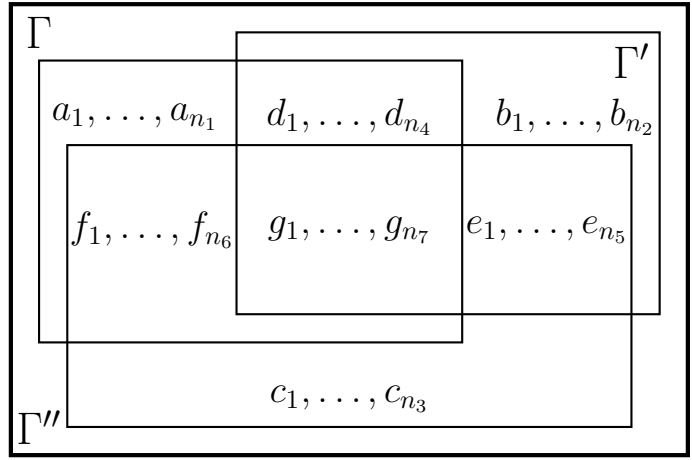
\begin{figure}[h]
\begin{center}
\begin{tikzpicture}[scale=0.75]
\draw [ultra thick] (0,0) rectangle (12,-8);
\draw [thick] (0.5,-1) rectangle (8,-6);
\draw [thick] (4,-0.5) rectangle (11.5,-5.5);
\draw [thick] (1,-2.5) rectangle (11,-7.5);
\node (Gamma) at (0.5,-0.5) {\LARGE $\Gamma$};
\node (Gamma) at (11,-1) {\LARGE $\Gamma'$};
\node (Gamma) at (0.5,-7.5) {\LARGE $\Gamma''$};
\node (a) at (2.2,-2) {\Large $a_1,\ldots,a_{n_1}$};
\node (b) at (10,-2) {\Large $b_1,\ldots,b_{n_2}$};
\node (c) at (6,-7) {\Large $c_1,\ldots,c_{n_3}$};
\node (d) at (6,-2) {\Large $d_1,\ldots,d_{n_4}$};
\node (e) at (9.5,-4) {\Large $e_1,\ldots,e_{n_5}$};
\node (f) at (2.5,-4) {\Large $f_1,\ldots,f_{n_6}$};
\node (g) at (6,-4) {\Large $g_1,\ldots,g_{n_7}$};
\end{tikzpicture}
\caption{Venn diagram for the proof of Lemma \ref{lem:disj_eli_is_transitive_over_chain}.}
\label{figure:Venn}
\end{center}
\end{figure}


\noindent Now suppose $<$ permits $\Gamma\orddeneq\Gamma'\orddeneq\Gamma''$, we write this out in terms of elements (Equations \ref{eq:first_ass_trans} and \ref{eq:second_ass_trans}). $\Gamma\orddeneq\Gamma'$ is equivalent to

\begin{align}\label{eq:first_ass_trans}
&\pair{\exists x\in\Gamma-\Gamma'}\pair{\forall y\in\Gamma'-\Gamma}\:x<y\nonumber\\
\Leftrightarrow&\pair{\exists x\in\set{a_1,\cdots,a_{n_1}}\cup\set{f_1,\cdots,f_{n_6}}}\pair{\forall y\in\set{b_1,\cdots,b_{n_2}}\cup\set{e_1,\cdots,e_{n_5}}}x<y\nonumber\\
\Leftrightarrow&\pair{\exists x\in\set{a_1,\:\cdots,\:a_{n_1}}\cup\set{f_1,\:\cdots,\:f_{n_6}}}\sqbra{\pair{\bigwedge_{i=1}^{n_2}x<b_i}\wedge\pair{\bigwedge_{j=1}^{n_5}x<e_j}}\nonumber\\
\Leftrightarrow&\bigvee_{k=1}^{n_1}\sqbra{\pair{\bigwedge_{i=1}^{n_2}a_k<b_i}\wedge\pair{\bigwedge_{j=1}^{n_5}a_k<e_j}}\vee\bigvee_{l=1}^{n_6}\sqbra{\pair{\bigwedge_{i=1}^{n_2}f_l<b_i}\wedge\pair{\bigwedge_{j=1}^{n_5}f_l<e_j}}.
\end{align}
\noindent Note that there are $(n_1+n_6)$ disjuncts in Equation \ref{eq:first_ass_trans}. Applying the same reasoning as in Equation \ref{eq:first_ass_trans}, we can see that $\Gamma'\orddeneq\Gamma''$ is equivalent to

\begin{align}\label{eq:second_ass_trans}
\bigvee_{k=1}^{n_2}\sqbra{\pair{\bigwedge_{i=1}^{n_3}b_k<c_i}\wedge\pair{\bigwedge_{j=1}^{n_6}b_k<f_j}}\vee\bigvee_{l=1}^{n_4}\sqbra{\pair{\bigwedge_{i=1}^{n_3}d_l<c_i}\wedge\pair{\bigwedge_{j=1}^{n_6}d_l<f_j}}.
\end{align}

\noindent There are $(n_2+n_4)$ disjuncts in Equation \ref{eq:first_ass_trans}. We need to show that $\Gamma\orddeneq\Gamma''$. By the same reasoning as Equations \ref{eq:first_ass_trans} and \ref{eq:second_ass_trans}, this is equivalent to
\begin{align}\label{eq:target}
\bigvee_{k=1}^{n_1}\sqbra{\pair{\bigwedge_{i=1}^{n_3}a_k<c_i}\wedge\pair{\bigwedge_{j=1}^{n_5}a_k<e_j}}\vee\bigvee_{l=1}^{n_4}\sqbra{\pair{\bigwedge_{i=1}^{n_3}d_l<c_i}\wedge\pair{\bigwedge_{j=1}^{n_5}d_l<e_j}}.
\end{align}

So, to prove Equation \ref{eq:target}, we need to show one of the disjuncts of Equation \ref{eq:target} i.e. for at least one of $1\leq k\leq n_1$ or $1\leq l\leq n_4$, we show either
\begin{align}\label{eq:answer_proof_trans}
\sqbra{\pair{\bigwedge_{i=1}^{n_3}a_k<c_i}\wedge\pair{\bigwedge_{j=1}^{n_5}a_k<e_j}}\text{ or }\sqbra{\pair{\bigwedge_{i=1}^{n_3}d_l<c_i}\wedge\pair{\bigwedge_{j=1}^{n_5}d_l<e_j}}
\end{align}
by establishing all of the conjuncts. Given $\Gamma\orddeneq\Gamma'\orddeneq\Gamma''$, we take the conjunction of Equations \ref{eq:first_ass_trans} and \ref{eq:second_ass_trans}, making $(n_1+n_6)(n_2+n_4)$ disjuncts, which is equivalent to the following expression:
\begin{align*}
&\set{\bigvee_{k=1}^{n_1}\sqbra{\pair{\bigwedge_{i=1}^{n_2}a_k<b_i}\wedge\pair{\bigwedge_{j=1}^{n_5}a_k<e_j}}\vee\bigvee_{l=1}^{n_6}\sqbra{\pair{\bigwedge_{i=1}^{n_2}f_l<b_i}\wedge\pair{\bigwedge_{j=1}^{n_5}f_l<e_j}}}\\
\wedge&\set{\bigvee_{k=1}^{n_2}\sqbra{\pair{\bigwedge_{i=1}^{n_3}b_k<c_i}\wedge\pair{\bigwedge_{j=1}^{n_6}b_k<f_j}}\vee\bigvee_{l=1}^{n_4}\sqbra{\pair{\bigwedge_{i=1}^{n_3}d_l<c_i}\wedge\pair{\bigwedge_{j=1}^{n_6}d_l<f_j}}}.
\end{align*}
As $\wedge$ and $\vee$ bi-distribute, we have four cases:
\begin{enumerate}
\item For some $1\leq k\leq n_1$ and $1\leq k'\leq n_2$, we have
\begin{align}
\pair{\bigwedge_{i=1}^{n_2}a_k<b_i}\wedge\pair{\bigwedge_{j=1}^{n_5}a_k<e_j}\wedge\pair{\bigwedge_{i'=1}^{n_3}b_{k'}<c_{i'}}\wedge\pair{\bigwedge_{j'=1}^{n_6}b_{k'}<f_{j'}}.
\end{align}
This means for some $1\leq k\leq n_1$, we have
\begin{align}\label{eq:1.1}
&\pair{\bigwedge_{j=1}^{n_5}a_k<e_j}\text{, and from}\\
&\pair{\bigwedge_{i=1}^{n_2}a_k<b_i}\wedge\pair{\bigwedge_{i'=1}^{n_3}b_{k'}<c_{i'}},\nonumber
\end{align}
that $1\leq k'\leq n_2$, and transitivity of $<$, we infer
\begin{align}\label{eq:1.2}
\pair{\bigwedge_{i=1}^{n_3}a_k<c_i}.
\end{align}
Equations \ref{eq:1.1} and \ref{eq:1.2} imply $\Gamma\orddeneq\Gamma''$.
\item For some $1\leq k\leq n_1$ and $1\leq l\leq n_4$, we have
\begin{align}\label{eq:case2}
\pair{\bigwedge_{i=1}^{n_2}a_k<b_i}\wedge\pair{\bigwedge_{j=1}^{n_5}a_k<e_j}\wedge\pair{\bigwedge_{i'=1}^{n_3}d_l<c_{i'}}\wedge\pair{\bigwedge_{j'=1}^{n_6}d_l<f_{j'}}
\end{align}
This case uses the assumption that $<$ is total.\footnote{It can be shown that if $<$ is not total, $\orddeneq$ is not transitive, see \cite[Lemma A.2]{APY_ArXiV:2016}.} The second and the third bracketed conjuncts in Equation \ref{eq:case2} are necessary but not sufficient to lead to $\Gamma\orddeneq\Gamma''$. Let $k_0$ be the witness to $1\leq k\leq n_1$ and $l_0$ be the witness to $1\leq l_0\leq n_4$. As $<$ is total, either $a_{k_0}<d_{l_0}$ or $d_{l_0}<a_{k_0}$ (remember all elements are distinct).
\begin{itemize}
\item If $a_{k_0}<d_{l_0}$ then $a_{k_0}<c_i$ for all $1\leq i\leq n_3$. Therefore, $\pair{\bigwedge_{i=1}^{n_3}a_{k_0}<c_i}$.
\item If $d_{l_0}<a_{k_0}$ then $d_{l_0}<e_j$ for all $1\leq j\leq n_5$. Therefore, $\pair{\bigwedge_{j=1}^{n_5}d_{l_0}<e_j}$.
\end{itemize}
In either case, $\Gamma\orddeneq\Gamma''$.
\item For some $1\leq l\leq n_6$ and $1\leq k\leq n_2$, we have

\begin{align}
\pair{\bigwedge_{i=1}^{n_2}f_l<b_i}\wedge\pair{\bigwedge_{j=1}^{n_5}f_l<e_j}\wedge\pair{\bigwedge_{i'=1}^{n_3}b_k<c_{i'}}\wedge\pair{\bigwedge_{j'=1}^{n_6}b_k<f_{j'}}
\end{align}
The irreflexivity of $<$ and the first and last bracketed conjuncts gives a contradiction when you run over all indices, so this case gives a contradiction.
\item For some $1\leq l\leq n_6$ and $1\leq l'\leq n_4$, we have
\begin{align}
\pair{\bigwedge_{i=1}^{n_2}f_l<b_i}\wedge\pair{\bigwedge_{j=1}^{n_5}f_l<e_j}\wedge\pair{\bigwedge_{i'=1}^{n_3}d_{l'}<c_{i'}}\wedge\pair{\bigwedge_{j'=1}^{n_6}d_{l'}<f_{j'}}
\end{align}
This case is similar to the first case -- we use transitivity to combine the second and last bracketed conjuncts. This infers the second conjunct of Equation \ref{eq:answer_proof_trans}, which means $\Gamma\orddeneq\Gamma''$.
\end{enumerate}
Therefore, in all cases, $\Gamma\orddeneq\Gamma''$. This shows $\orddeneq$ is transitive on $\powfin\pair{P}$.

To show trichotomy, let $\Gamma,\:\Gamma'\in\powfin(P)$ be arbitrary. We start by assuming $\Gamma\neq\Gamma'$ and show exactly one of $\Gamma\orddeneq\Gamma'$ or $\Gamma'\orddeneq\Gamma$ is true. From Equation \ref{eq:disj_eli}, we consider the symmetric difference $\Gamma\ominus\Gamma'$. The set $\ang{\Gamma\ominus\Gamma',\:<}\:\subseteq\ang{P,\:<}$ is also a finite strict toset. This means there must exist a $<$-least element $x_0\in\Gamma\ominus\Gamma'$, say. There are two mutually exclusive possibilities. If $x_0\in\Gamma-\Gamma'$, then $\Gamma\orddeneq\Gamma'$. If $x_0\in\Gamma'-\Gamma$, then $\Gamma'\orddeneq\Gamma'$. This establishes trichotomy. Therefore, $\ang{\powfin\pair{P},\orddeneq}$ is a strict chain.
\end{proof}

\noindent Therefore, given the strict toset $\ang{\relsymb_d,<_D}$, $\ang{\powfin\pair{\relsymb_d},\ordneq_{SP}}$ is also a strict toset.

\begin{lem}\label{lem:SP_total_preorder_when_total}
The argument preference $\precsim_{SP}$ is a total preorder on $\alg$.
\end{lem}
\begin{proof}
We instantiate $\ang{P,<}$ in Lemma \ref{lem:disj_eli_is_transitive_over_chain} to $\ang{\relsymb_d,<_{SP}}$. This is valid because by Equation \ref{eq:SP_order} and the discussion aftewards, $<_{SP}$ is a strict total order on $\relsymb_d$. Further, Equation \ref{eq:SP_set_comparison} is Equation \ref{eq:disj_eli} with $<_{SP}$ instead of $<_D$. Therefore, $\ang{\powfin\pair{\relsymb_d},\ordneq_{SP}}$ is a strict toset by Lemma \ref{lem:disj_eli_is_transitive_over_chain}. By Equation \ref{eq:SP_arg_pref}, $\prec_{SP}$ is a strict total order on $\alg$, and $\precsim_{SP}$ (Equation \ref{eq:SP_arg_pref_non_strict}) is a total preorder on $\alg$.
\end{proof}

\noindent By Lemma \ref{lem:SP_total_preorder_when_total}, if two arguments $A$ and $B$ satisfy $A\not\prec_{SP}B$, then $B\precsim_{SP}A$. We demonstrate the features of $<_{SP}$ and $\prec_{SP}$ with Examples \ref{eg:not_disj_eli_WLP} and \ref{eg:always_blocked_default}.

\begin{eg}\label{eg:not_disj_eli_WLP}
Suppose that instead of respecting the logical structure of the defeasible rules with $<_{SP}$, we use an argument preference relation $\prec$ based on $\orddeneq$ (Equation \ref{eq:disj_eli}) instead of $\ordneq_{SP}$, i.e. replace $\ordneq_{SP}$ in Equation \ref{eq:SP_arg_pref_non_strict} with $\orddeneq$. Now consider the following LPDT. Let $T^+$ have $W=\es$ and $D=\set{d_k}_{k=1}^5$ where
\begin{align*}
d_1:=\frac{:c_1}{c_1},\:d_4:=\frac{c_3:c_4}{c_4},\:d_3:=\frac{:c_3}{c_3},\:d_2:=\frac{c_1:c_2}{c_2},\:d_5:=\frac{c_1:\neg(c_2\wedge c_4)}{\neg(c_2\wedge c_4)},
\end{align*}
\noindent such that $d_1<^+ d_4<^+ d_3<^+ d_2<^+ d_5$. Our PDE is constructed in the usual manner starting from $E_0=Th(\es)$ by Equation \ref{eq:ext_base}. Equation \ref{eq:ext_ind} gives the order of application of the defaults:
\begin{align}\label{eq:order_of_adding_rules}
E_1=E_0+c_3,\:E_2=E_1+c_4,\:E_3=E_2+c_1,\:E_4=E_3+\neg(c_2\wedge c_4),
\end{align}
with $E_k=E_4$ for all $k\geq 5$. As $\neg(c_2\wedge c_4)\equiv(\neg c_2\vee\neg c_4)$, along with $c_4$ (from $d_4$), we have $\neg c_2$, which blocks $d_2$. The unique PDE from this LPDT is $E:=Th(\set{c_1,\neg c_2,c_3,c_4})$. Now consider the corresponding arguments following our instantiation. We have the defeasible rules\footnote{We define, for $1\leq i\leq 5$, $r_i:=f\pair{d_i}$, by Equation \ref{eq:bij_defaults_def_rules}.} $r_1<_Dr_4<_Dr_3<_Dr_2<_Dr_5$. The relevant arguments and sets of defeasible rules are
\begin{align*}
A&:=[[\Rightarrow c_1]\Rightarrow c_2],DR(A)=\set{r_1,r_2}\\
B&:=[[\Rightarrow c_3]\Rightarrow c_4],DR(B)=\set{r_3,r_4}\\
C&:=[[\Rightarrow c_1]\Rightarrow\neg(c_2\wedge c_4)],DR(C)=\set{r_1,r_5},\\
D&:=[B,C\to\neg c_2],DR(D)=\set{r_1,r_3,r_4,r_5}.
\end{align*}

\noindent We illustrate these arguments in Figure \ref{figure:ad_nauseam}. Our convention for diagrams is that broken arrows represent defeasible rules, and solid arrows represent strict rules. Solid vertical lines spanning the length of arguments label those arguments. In the diagrams of this paper, defeasible rules with empty antecedent have the symbol $\top$ as a placeholder for their antecedent.

\begin{figure}[h]
\begin{center}
\begin{tikzpicture}[scale = 1]
\node (Ta) at (0,0) {$ \top $};
\node (c11) at (0,-1) {$ c_1 $};
\draw [thick, dashed, ->] (Ta) -- (c11);
\node (r1) at (0.3,-0.5) {$r_1$};
\node (c2) at (0,-2) {$ c_2 $};
\draw [thick, dashed, ->] (c11) -- (c2);
\node (r2) at (0.3,-1.5) {$r_2$};
\draw [ultra thick] (-0.3,0.3)--(-0.3,-2.3);
\node (A) at (-0.6,-1) {$A$};

\node (Tb) at (2,0) {$ \top $};
\node (c3) at (2,-1) {$ c_3 $};
\draw [thick, dashed, ->] (Tb) -- (c3);
\node (r3) at (2.3,-0.5) {$r_3$};
\node (c4) at (2,-2) {$ c_4 $};
\draw [thick, dashed, ->] (c3) -- (c4);
\node (r4) at (2.3,-1.5) {$r_4$};
\draw [ultra thick] (1.7,0.3)--(1.7,-2.3);
\node (B) at (1.4,-1) {$B$};

\node (Tc) at (4.5,0) {$ \top $};
\node (c12) at (4.5,-1) {$ c_1 $};
\draw [thick, dashed, ->] (Tc) -- (c12);
\node (r1) at (4.8,-0.5) {$r_1$};
\node (c5) at (4.5,-2) {$ \neg(c_2\wedge c_4) $};
\draw [thick, dashed, ->] (c12) -- (c5);
\node (r5) at (4.8,-1.5) {$r_5$};
\draw [ultra thick] (3.7,0.3)--(3.7,-2.3);
\node (C) at (3.4,-1) {$C$};
\draw [thick] (c4)--(2,-3);
\draw [thick] (c5)--(4.5,-3);
\draw [thick] (2,-3)--(4.5,-3);
\node (nc2) at (3.25,-4) {$ \neg c_2 $};
\draw [thick,->] (3.25,-3)--(nc2);
\draw [ultra thick] (5.5,0.3)--(5.5,-4.2);
\node (D) at (5.8,-2) {$D$};
\end{tikzpicture}
\caption{The arguments of Example \ref{eg:not_disj_eli_WLP}.}
\label{figure:ad_nauseam}
\end{center}
\end{figure}
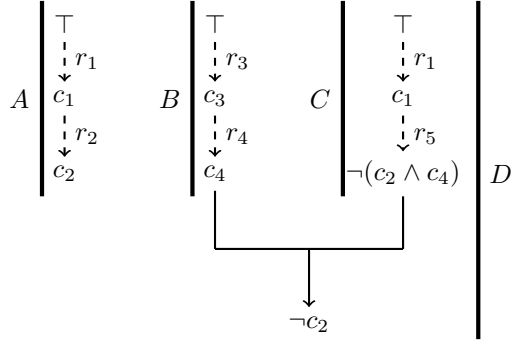

For the stable extension to correspond to the PDL extension, the desired stable extension contains the arguments $D,\:B,\:C,\:[\Rightarrow c_3],\:[\Rightarrow c_1]$, the conclusions of which are, respectively, $\neg c_2,\:c_4,\:\neg\pair{c_2\wedge c_4},\:c_3,\:c_1$, which under deductive closure, corresponds to $E$. However, this would require $D\defeat A$, which means, by Equation \ref{eq:ASPIC+_general_defeat}, $D\attk A$ and $D\not\prec A$. Clearly, $D\attk A$ on $A$. However, it is not the case that $r_2$ is the $<_D$-least defeasible rule, so $D\prec A$. Therefore, this argument preference relation does not generate the corresponding stable extension to $E$.

Suppose now that we do respect the logical structure of the rules and use $\prec_{SP}$ as our argument preference (Equation \ref{eq:SP_arg_pref}). By applying Equations \ref{eq:SP_ord_def} and \ref{eq:SP_order}, we can show that $a_1=r_3$, $a_2=r_4$, $a_3=r_1$, $a_4=r_5$ and $a_5=r_2$. The structure preference order is $r_2<_{SP}r_5<_{SP}r_1<_{SP}r_4<_{SP}r_3$. Notice that this is precisely the order in which the corresponding normal defaults are added in PDL, as Equation \ref{eq:order_of_adding_rules} shows. It is easy to show that the corresponding stable extension under the argument preference $\prec_{SP}$ corresponds to the PDL inference, because $r_2$ is now $<_{SP}$-least, so $D\not\prec_{SP}A$. Therefore $A\prec_{SP}D$, so $A\defeat D$.
\end{eg}

\begin{eg}\label{eg:always_blocked_default}
However, $<_{SP}$ does not necessarily follow the PDL order of the application of defaults. Consider $\ang{\set{d_1,d_2},\:\set{a},\:<^+}$ with $d_1:=\frac{a:\neg a}{\neg a}$ and $d_2:=\frac{:b}{b}$ such that $d_2<^+ d_1$. We have $E=Th\pair{\set{a,b}}$, where $d_1$ is blocked by $W$, so $d_2$ is the only default added. In argumentation, we have $\mathcal{K}_n=\set{a}$, $r_1:=(a\Rightarrow\neg a)$ and $r_2:=(\Rightarrow b)$, such that $r_2<_D r_1$. The arguments are $A_0:=[a]$, $A_1:=[A_0\Rightarrow\neg a]$ and $B:=[\Rightarrow b]$. Applying Equation \ref{eq:SP_ord_def}, we have $r_2<_{SP}r_1$, which clearly is not the order of how the corresponding defaults are added in PDL. Yet the correspondence still holds, since $A_0\defeat A_1$ because $A_0$ is strict and strict arguments always defeat any non-strict argument they attack, so the stable extension is the strict extension of $\set{A_0,\:B}$, the conclusion set of which (after deductive closure) is the extension of the underlying LPDT.
\end{eg}


We have now defined the structure-preference argument preference relation $\precsim_{SP}$. Given an LPDT $T^+$, we denote its \textit{attack graph} to be $AG\pair{T^+}:=\ang{\alg,\attk,\precsim_{SP}}$, and its \textit{defeat graph} to be $DG\pair{T^+}:=\ang{\alg,\defeat}$, where $\defeat$ is defined by Equation \ref{eq:ASPIC+_general_defeat} with $\precsim$ equal to $\precsim_{SP}$.

\subsection{The Representation Theorem}\label{sec:rep_thm}

\noindent In this section we state and prove the representation theorem (Theorem \ref{thm:rep_thm}), which guarantees that the inferences with argumentation semantics under the preference $\precsim_{SP}$ correspond exactly to the inferences in PDL.

\subsubsection{Non-Blocked Defaults}

\noindent We introduce some concepts to help prove the representation theorem. Let $T$ be a PDT and $E=\bigcup_{i\in\nat}E_i$ one of its extensions generated from the linearisation $<^+\supseteq<$. The \emph{set of generating defaults (w.r.t. $<^+$), $GD(<^+)$}, is defined as
\begin{align}\label{eq:GD}
GD_i(<^+)&:=\set{d\in D\:\vline\:\text{$d$ is $<^+$-greatest active in $E_i$}},\nonumber\\
GD(<^+)&:=\bigcup_{i\in\nat}GD_i(<^+)\subseteq D.
\end{align}
Intuitively, this is the set of defaults applied to calculate $E$ following the priority $<^+$. However, the same $E$ can be generated by distinct total orders.
\begin{eg}
Consider the PDT $\ang{\set{\frac{a:c}{c},\:\frac{b:c}{c}},\:\set{a,\:b},\:\es}$. We have two possible linearisations $\frac{a:c}{c}<_1^+\frac{b:c}{c}$ and $\frac{b:c}{c}<_2^+\frac{a:c}{c}$. By Footnote \ref{fn:dual_priority_PDL} (page \pageref{fn:dual_priority_PDL}) we have $GD(<_1^+)=\set{\frac{b:c}{c}}$ and $GD(<_2^+)=\set{\frac{a:c}{c}}$, which are not equal, even though both linearisations give the same extension $E=Th\pair{\set{a,\:b,\:c}}$. In both cases, the default in $D-GD\pair{<_i^+}$ (for $i=1,2$) is not active not because it is blocked by $\neg c$, but rather because it adds no new information.
\end{eg}
\noindent We wish to distinguish between inactive defaults that conflict with something known and inactive defaults that do not add any new information. We call a default $\frac{\theta:\phi}{\phi}$ \emph{semi-active (in $S\subseteq\LSent$)} iff $\sqbra{\theta\in S,\:\neg\phi\notin S,\:\phi\in S}$. Let $<^+\:\supseteq\:<$ be total and which generates the extension $E$ (Equation \ref{eq:ext_ind}). The \emph{set $SAD\pair{<^+}$ of semi-active defaults with respect to the linearisation $<^+$} is defined as
\begin{align}\label{eq:SAD}
\set{d\in D\:\vline\:\text{ $d$ is semi-active w.r.t. $E$, which is generated by $<^+$}}.
\end{align}
Semi-active defaults add no new information. The \emph{set of non-blocked defaults} is
\begin{align}\label{eq:NBD}
NBD(<^+):=GD(<^+)\cup SAD(<^+)\subseteq D.
\end{align}
\begin{lem}
If $<^+$ generates the PDE $E$, then
\begin{align}\label{eq:lem_NBD}
NBD(<^+):=\set{\frac{\theta:\phi}{\phi}\in D\:\vline\:\theta\in E,\:\neg\phi\notin E}.
\end{align}
\end{lem}
\begin{proof}
It is sufficient to show Equation \ref{eq:NBD} (with Equations \ref{eq:GD} and \ref{eq:SAD}) is the same as the right hand side of Equation \ref{eq:lem_NBD}. Let $<^+$ generate the extension $E$ and, for notational convenience, we suppress the argument ``$<^+$'' in the sets for this proof.\footnote{For example, instead of writing ``$GD\pair{<^+}$'' we write ``$GD$''.} ($\Rightarrow$) Case 1: Assume $d\in SAD$, then $Ante(d)\subseteq E,\:\neg Conc(d)\notin E$ and $Conc(d)\in E$. 
\begin{align}
d\in&\set{d'\in D\:\vline\:Ante(d')\subseteq E,\:\neg Conc(d')\notin E}\text{ and hence}\label{eq:NBD_proof_inter}\\
SAD\subseteq&\set{d\in D\:\vline\:Ante(d)\subseteq E,\:\neg Conc(d)\notin E}.\label{eq:SAD_subset_alternative_form}
\end{align}

\noindent Case 2: Now assume $d\in GD$, which is equivalent to $\pair{\exists i\in\nat}d\in GD_i$. This is equivalent to $\pair{\exists i\in\nat}\sqbra{Ante(d)\subseteq E_i,\: Conc(d)\notin E_i,\:\neg Conc(d)\notin E_i}$, which is equivalent to
\begin{align}\label{eq:intermediate_d_is_active}
& Ante(d)\subseteq E_{j_0},\:Conc(d)\notin E_{j_0},\:\neg Conc(d)\notin E_{j_0}\text{ $j_0$ witness to $i$}.
\end{align}
This implies $Ante(d)\subseteq E_{j_0},\:\neg Conc(d)\notin E_{j_0}$, and it follows that $Ante(d)\subseteq E,\:\neg Conc(d)\notin E_{j_0}$. Clearly, this means $Ante(d)\subseteq E$. Now assume for contradiction that $\neg Conc(d)\in E$, which means there is some $i_0\in\nat$ such that $\neg Conc(d)\in E_{i_0}$. What is the relationship between $i_0$ and $j_0$? There are three possibilities:
\begin{itemize}
\item $i_0=j_0$ would mean $\neg Conc(d)\notin E_{i_0}$ and $\neg Conc(d)\in E_{i_0}$ -- contradiction.
\item $i_0<j_0$: We have $\neg Conc(d)\notin E_{j_0}$ and $\neg Conc(d)\in E_{i_0}$, which is also impossible because the $E_i$'s form an ascending chain, so $E_{i_0}\subseteq E_{j_0}$. Therefore, we get $\neg Conc(d)\in E_{j_0}$ and $\neg Conc(d)\notin E_{j_0}$.
\item $i_0>j_0$: We have $\neg Conc(d)\notin E_{j_0}$ and $\neg Conc(d)\in E_{i_0}$. From Equation \ref{eq:intermediate_d_is_active}, we have that $d$ is active in $E_{j_0}$, hence $Conc(d)\in E_{j_0+1}\subseteq E_{i_0}$, which makes $\neg Conc(d)\in E_{i_0}$ impossible because the $E_i$'s are consistent.
\end{itemize}

\noindent Therefore, $\neg Conc(d)\notin E$. So we have $Ante(d)\subseteq E$ and $\neg Conc(d)\notin E$. Therefore, Equation \ref{eq:NBD_proof_inter} is true for this case and we have
\begin{align}\label{eq:GD_subset_alternative_form}
GD\subseteq\set{d\in D\:\vline\:Ante(d)\subseteq E,\:\neg Conc(d)\notin E}.
\end{align}

\noindent We can take the union of Equations \ref{eq:SAD_subset_alternative_form} and \ref{eq:GD_subset_alternative_form} to get
\begin{align}\label{eq:half_the_alternative_form1}
GD\cup SAD\subseteq\set{d\in D\:\vline\:Ante(d)\subseteq E,\:\neg Conc(d)\notin E}.
\end{align}

($\Leftarrow$) Now assume $d\in\set{d'\in D\:\vline\:Ante(d')\subseteq E,\:\neg Conc(d')\notin E}$, which means $Ante(d)\subseteq E$ and $\neg Conc(d)\notin E$. We have, for some $i_0\in\nat$,

\begin{align*}
\Leftrightarrow& Ante(d)\subseteq E_{i_0},\:\pair{\forall j\in\nat}\neg Conc(d)\notin E_j\\
\Leftrightarrow& Ante(d)\subseteq E_{i_0},\:\neg Conc(d)\notin E_{i_0},\:\pair{\forall j\in\nat-\set{i_0}}\neg Conc(d)\notin E_j\\
\Leftrightarrow&\pair{\forall j\in\nat-\set{i_0}}\neg Conc(d)\notin E_j\text{ and }\\
&[\pair{Ante(d)\subseteq E_{i_0},\:\neg Conc(d)\notin E_{i_0},\:Conc(d)\in E_{i_0}}\text{ or }\\
&\pair{Ante(d)\subseteq E_{i_0},\:\neg Conc(d)\notin E_{i_0},\:Conc(d)\notin E_{i_0}}]
\end{align*}
\begin{align*}
\Leftrightarrow&\pair{\forall j\in\nat-\set{i_0}}\neg Conc(d)\notin E_j\text{ and }\\
&[\pair{Ante(d)\subseteq E_{i_0},\:\neg Conc(d)\notin E_{i_0},\:Conc(d)\in E_{i_0}}\text{ or }d\in GD_{i_0}\\
\Rightarrow&\pair{\forall j\in\nat-\set{i_0}}\neg Conc(d)\notin E_j\text{ and }\\
&[\pair{Ante(d)\subseteq E_{i_0},\:\neg Conc(d)\notin E_{i_0},\:Conc(d)\in E_{i_0}}\text{ or }d\in GD\\
\Rightarrow& d\in GD\text{ or }[Ante(d)\subseteq E_{i_0},\:\neg Conc(d)\notin E_{i_0},\:Conc(d)\in E_{i_0}\text{ and }\\
&\pair{\forall j\in\nat-\set{i_0}}\neg Conc(d)\notin E_j]\\
\Rightarrow& d\in GD\text{ or }\sqbra{Ante(d)\subseteq E\text{ and }\pair{\forall j\in\nat}\neg Conc(d)\notin E_j}\\
\Rightarrow& d\in GD\text{ or }\sqbra{Ante(d)\subseteq E\text{ and }\neg Conc(d)\notin E}\Leftrightarrow d\in GD\cup SAD.
\end{align*}
Therefore, we have
\begin{align}\label{eq:half_the_alternative_form2}
\set{d\in D\:\vline\:Ante(d)\subseteq E,\:\neg Conc(d)\notin E}\subseteq GD\cup SAD.
\end{align}

\noindent The result follows from Equations \ref{eq:half_the_alternative_form1} and \ref{eq:half_the_alternative_form2}.
\end{proof}
\noindent Given $E$, $NBD(<^+)$ is uniquely determined by Equation \ref{eq:lem_NBD}, so we will write $NBD(E)$ instead. Equation \ref{eq:lem_NBD} adapts Reiter's idea of a \emph{generating default} \cite[page 92, Definition 2]{Reiter:80} to PDL.

We use these concepts to show that the rearrangement of rules $<_D\:\mapsto\:<_{SP}$, as defined in Equation \ref{eq:SP_order}, does not change the extension of the LPDT. This is because the manner through which $<_{SP}$ incorporates the argument structure captures the idea of Equation \ref{eq:ext_ind}, which is the method of how PDL incorporates both structure and preference when choosing the ``$<^+$-most active'' default.

\begin{lem}\label{lem:LPDT_SP_order_keeps_extension_same}
Let $T:=\ang{D,W,<^+}$ and $T':=\ang{D,W,<^{+'}}$ be two LPDTs such that $<^+\:\cong\:<_D\:\mapsto\:<_{SP}\cong <^{+'}$,\footnote{This means that the chain $\ang{D,<^+}$ is order isomorphic to $\ang{\relsymb_d,<_D}$ as described by Equations \ref{eq:bij_defaults_def_rules} and \ref{eq:def_rules_pref_order}. Then we calculate $<_{SP}$ from $<_D$ as described in Section \ref{sec:pref} and form a new chain $\ang{D,<^{+'}}$, which is order isomorphic to $\ang{\relsymb_d,<_{SP}}$.} then both $T$ and $T'$ have the same extension $E$.
\end{lem}
\begin{proof}
Let $E$ be the unique extension of $T$, and $E'$ be the unique extension of $T'$. To show $E=E'$, we need to show that they have the same generating defaults, i.e. $GD\pair{<^+}=GD\pair{<^{+'}}$. As $<^{+'}\:\cong\:<_{SP}$, and $<^+\:\cong\:<_D\:\mapsto\:<_{SP}$, the rearrangement $<_D\:\mapsto\:<_{SP}$ will always choose the $<_D$-greatest active defeasible rule for $a_1$ in Equation \ref{eq:SP_ord_def}, the second $<_D$-greatest active defeasible rule for $a_2$... etc. until all defeasible rules are rearranged, but the defeasible rules corresponding to the generating defaults of $<^{+'}$ will always be chosen first in the rearrangement, therefore $GD\pair{<^+}=GD\pair{<^{+'}}$ and hence the result follows.
\end{proof}

\subsubsection{Existence and Uniqueness of Stable Extensions}

\noindent Let $T^+$ be an LPDT. In this section we show that its defeat graph $DG\pair{T^+}$ has a unique stable extension. We propose an algorithm that imitates how PDL extensions are constructed over an LPDT (Equation \ref{eq:ext_ind}). Given $S\subseteq\alg$, $r\in\relsymb_d$, the definition of $Args\pair{\:\cdot\:}$ (Equation \ref{eq:def_Args(R)}) and $<_{SP}$ we define $\oplus$ as $S\oplus r:=Args(DR(S)\cup\set{r})$, i.e. we close $S$ under all arguments over all strict rules, all defeasible rules in $S$, \textit{and} the addition of a new defeasible rule $r$.

Consider Algorithm \ref{alg:gen_stab_ext}, which takes as input the attack graph $AG\pair{T^+}$ of an LPDT $T^+$, and the strict chain of defeasible rules under the SP order $\ang{\relsymb_d,<_{SP}}$. The output is a set of arguments $S\subseteq\alg$. The formal definition is:\footnote{This is a brute-force definition used to prove that stable extensions exist and are unique in such defeat graphs.}

\begin{algorithm}[h]
\begin{algorithmic}[1]
\Function{GenerateStableExtension}{$\ang{\alg,\:\attk,\:\precsim_{SP}}$, $\ang{\relsymb_d,<_{SP}}$}
  \State $S\gets\set{\text{all strict arguments in $\alg$}}$\label{alg_line:input_strict_args}
  \For{$r\in\relsymb_d$ from $<_{SP}$-greatest to $<_{SP}$-smallest}
    \If{$S\oplus r$ has no attacks, $\pair{S\oplus r}^2\cap\attk\:=\:\es$,}\label{alg_line:cond_start}
    \State {$S\gets S\oplus r$\label{alg_line:cond_end}}
    \EndIf
  \EndFor
  \Return $S$
\EndFunction
\end{algorithmic}
\caption{Generating a Stable Extension}
\label{alg:gen_stab_ext}
\end{algorithm}

\noindent Algorithm \ref{alg:gen_stab_ext} first creates the largest possible set of undefeated arguments that do not attack each other. This includes all strict arguments as they are never defeated nor do they attack each other, and possibly some undefeated defeasible arguments. Then, the algorithm includes the defeasible rules from most to least preferred under $<_{SP}$ and tests whether the resulting arguments that are constructed by the inclusion of such a defeasible rule attack each other in the sense of Equation \ref{eq:attack} (Lines \ref{alg_line:cond_start}--\ref{alg_line:cond_end}). Note that the resulting attack must originate from the arguments having at most the defeasible rules added so far. As $<_{SP}$ is total, all defeasible rules are considered, and the result includes as many defeasible rules as possible such that the result has no attacks. Adding the rules in the order of $<_{SP}$ while ensuring no attacks mimics the condition of Equation \ref{eq:ext_ind}. It is clear from the algorithm that $S$ exists and is unique given the input, as it is of the form $Args(R)$ for some $R\subseteq\relsymb_d$ (Equation \ref{eq:def_Args(R)}). We show $S$ is a stable extension.

\begin{lem}\label{lem:alg1_output_is_cf}
The output $S$ of Algorithm \ref{alg:gen_stab_ext} is cf (conflict free).
\end{lem}
\begin{proof}
cf is guaranteed by the consistency of $\mathcal{K}_n$ and that defeasible rules $r\in\relsymb_d$ are only added if the resulting arguments do not attack each other (Lines \ref{alg_line:cond_start} - \ref{alg_line:cond_end}). Therefore, by Equation \ref{eq:ASPIC+_general_defeat}, $S$ contains no defeats and must be cf.
\end{proof}

Note that the setup of the algorithm prevents not just defeats but attacks from appearing in $S$ (Line \ref{alg_line:cond_start}). Normally, this is not sufficient to guarantee that $Conc\pair{S}$ is consistent in FOL.

\begin{eg}
Consider $S=\set{\sqbra{\Rightarrow a},\sqbra{\Rightarrow b},\sqbra{\Rightarrow\neg(a\wedge b)}}\subseteq\alg$. There are no attacks in $S$ because attacks are defined syntactically (Equation \ref{eq:attack}, page \pageref{eq:attack}). However, $Conc\pair{S}=\set{a,b,\neg\pair{a\wedge b}}$ is clearly inconsistent in FOL.
\end{eg}

\noindent We now show that if $S$ has no attacks then $Conc\pair{S}$ is consistent in FOL. 

\begin{lem}\label{lem:acf_S_consistent}
Let $S$ be the output of Algorithm \ref{alg:gen_stab_ext}. If $S^2\cap\attk=\es$ then $Conc\pair{S}$ is consistent.
\end{lem}
\begin{proof}
By construction and Lemma \ref{lem:Args(R)_props}, $Conc\pair{S}$ is deductively closed. Assume for contradiction that $Conc\pair{S}$ is inconsistent, then $\theta,\:\neg\theta\in Conc\pair{S}$ for some $\theta\in\LSent$. Hence there are $A,B\in S$ such that $Conc(A)=\theta$ and $Conc(B)=\neg\theta$. If at least one of $TopRule(A)$ or $TopRule(B)$ are defeasible then at least one of $A\attk B$ and $B\attk A$ is the case, hence $S^2\cap\attk\neq\es$ -- contradiction.

Now consider the case where $TopRule(A),\:TopRule(B)\in\relsymb_s$ are both strict. As $W=\mathcal{K}_n$ is consistent and the rules in $\relsymb_s$ are sound, if $A$ and $B$ are both strict then it cannot be the case they have contradictory conclusions. Therefore, at least one of $A$ and $B$ are defeasible. WLOG say $A$ is defeasible. Suppose by construction $A=\sqbra{A_1,A_2,\ldots,A_i,\ldots,A_n\to\theta}$ and $B=\sqbra{B_1,B_2,\ldots,B_m\to\neg\theta}$, where $A_i\subarg A$ is defeasible with conclusion $a_i$ (Section \ref{sec:rev_ASPIC+}). By closure under transposition of $\relsymb_s$ and the properties of $S$, we can construct the argument $B^{\pair{i}}:=\sqbra{A_1,A_2,\ldots,A_{i-1},B,A_{i+1},\ldots,A_n\to\neg a_i}$, and by Lemma \ref{lem:Args(R)_props}, $B^{(i)}\in S$. If $TopRule\pair{A_i}$ is defeasible, then $B^{\pair{i}}\attk A_i$ and hence $S^2\cap\attk\neq\es$ -- contradiction, so $TopRule\pair{A_i}$ is not defeasible. As $A_i$ is defeasible we choose some subargument $A'_i\subarg A_i$ and repeat the above line of reasoning for $B^{\pair{i}}$ and $A'_i$. As all arguments are well-founded, this line of reasoning must terminate at some subargument of $A_i$ whose top rule is defeasible, and hence $S^2\cap\attk\neq\es$ -- contradiction. The result follows.
\end{proof}

\begin{lem}\label{lem:alg1_output_defs_all_outside}
The set $S$ defeats all arguments outside of itself.
\end{lem}
\begin{proof}
Let $R:=DR(S)$. Let $B\notin S$ be an arbitrary argument outside of $S$. We show there is an $A\in S$ such that $A\defeat B$. Given that $B\notin S$, there must be some rule $r\in DR(B)-R$ that causes $S$ to attack the subargument of $B$ with top rule $r$, according to Algorithm \ref{alg:gen_stab_ext}, Line \ref{alg_line:cond_start}. Let $B'\subarg B$ such that $TopRule(B')=r$. Let $A\in S$ be the attacker of $B'$ at $r$, such that $Conc(A)=\overline{Cons(r)}$.\footnote{Note that $A$ is appropriately chosen such that $Conc(A)=\overline{Cons(r)}$ is syntactic equality. This is always possible because $\relsymb_s$ has all rules of proof of FOL. Therefore, if an argument $C$ concludes $\theta$, and we would want it to conclude $\phi$, where $\phi\equiv\theta$, we can just append the strict rule $(\theta\to\phi)\in\relsymb_s$ to $C$ to create a new argument $D$ that concludes $\phi$.} This means $A\attk B'$ and hence $A\attk B$. There are two possibilities: either this rule $r\in DR(B)$ is $<_{SP}$-greatest, or it is not.

If $r$ is $<_{SP}$-greatest, then $Args(\es)\oplus r$ contains attacking arguments, so $A$ must be strict and hence $A\defeat B$. If $r$ is not $<_{SP}$-greatest, then consider the strict $<_{SP}$-upper-set of $r$ in $\ang{\relsymb_d,\:<_{SP}}$, $T:=\set{r'\in\relsymb_d\:\vline\:r<_{SP}r'}\neq\es$. There are two sub-possibilities: either $T\cap R=\es$ or $T\cap R\neq\es$. If $T\cap R=\es$, then adding $r$ to $S$ will create an attack by Algorithm \ref{alg:gen_stab_ext}, Line \ref{alg_line:cond_start}, and this attack must originate from some $A\in Args(\es)$ because no rule $<_{SP}$-larger than $r$ is used in the arguments of $S$, hence $A\defeat B$.

If $T\cap R\neq\es$, then adding $r$ to $S$ means its attacker $A\attk B'$ is in $Args\pair{T\cap R}$.\footnote{We have $A\in Args\pair{T\cap R}$ because as Algorithm \ref{alg:gen_stab_ext} adds the rules one by one according to $<_{SP}$, if adding $r$ to the rules in $S$ and then creating all arguments (with all strict rules) creates an attack, then this attack must be due to some argument whose defeasible rules are amongst $T\cap R$. This is because at the point where the algorithm excludes $r$, any defeasible arguments constructed then can only have their rules from $T$.} Either $A$ is strict or not strict (i.e. defeasible). If it is strict, then $A\defeat B$ as before. If it is not strict, i.e. $\es\neq DR(A)\subseteq T\cap R$, then by definition $\pair{\forall s\in T}r<_{SP}s$. As $DR(A)\subseteq T\cap R$, we must also have $\pair{\forall s\in DR(A)}r<_{SP} s$. Therefore, there is an $r\in DR(B')-DR(A)$ such that for all rules in $DR(A)$, and hence $DR(A)-DR(B')$, $r<_{SP} s$. By Equation \ref{eq:SP_arg_pref}, we conclude that $B'\prec_{SP} A$, and hence $A\defeat B'$. Therefore, by definition of $\defeat$ and $\subarg$, $A\defeat B$.
\end{proof}

\begin{thm}\label{thm:well_defined_stb_ext}
The output of Algorithm \ref{alg:gen_stab_ext}, $S$, is a stable extension of $DG\pair{T^+}$.
\end{thm}
\begin{proof}
Immediate from Lemmata \ref{lem:alg1_output_is_cf} and \ref{lem:alg1_output_defs_all_outside}.
\end{proof}

We also have a useful property relating the presence of an argument in a stable extension with its rules, which is independent of Algorithm \ref{alg:gen_stab_ext}.

\begin{lem}\label{lem:rules_in_stb_ext}
For a LPDT $T^+$ and defeat graph $DG\pair{T^+}$, if $\ext$ is a stable extension of $DG\pair{T^+}$, then $A\in\ext\Leftrightarrow DR(A)\subseteq DR\pair{\ext}$.
\end{lem}
\begin{proof}
($\Rightarrow$) If $A\in\ext$ then $DR(A)\subseteq\bigcup_{A\in\ext}DR(A)=DR(\ext)$ trivially. ($\Leftarrow$, contrapositive) If $A\notin\ext$, then $\ext\defeat A$ at some $A'\subarg A$. Let $r:=TopRule\pair{A'}$. Assume for contradiction that $r\in DR(\ext)$, then $\pair{\exists B\in\ext}r\in DR(B)$, so $\ext\defeat B$ -- contradiction, as $\ext$ is cf. Therefore, $r\notin DR\pair{\ext}$. But as $r\in DR(A)$, $DR(A)\not\subseteq DR\pair{\ext}$.
\end{proof}

We have shown that given $T^+$ and $AG\pair{T^+}$, Algorithm \ref{alg:gen_stab_ext} gives a unique output that is a stable extension (Theorem \ref{thm:well_defined_stb_ext}). We now show that this is the \textit{only} stable extension that $DG\pair{T^+}$ can have.

\begin{thm}\label{thm:total_still_has_unique_stable_extension}
Let $\ext$ be the stable extension that is the output of Algorithm \ref{alg:gen_stab_ext}. This is the unique stable extension of $DG\pair{T^+}$.
\end{thm}
\begin{proof}
Given $DG\pair{T^+}=\ang{\alg,\defeat}$, let $\ext$ be the output of Algorithm \ref{alg:gen_stab_ext}, and assume for contradiction that $\ext'\neq\ext$ is some other stable extension of $DG\pair{T^+}$. Let $A'_1\in\ext'-\ext$. There is an argument $A_2\in\ext-\ext'$ such that $A_2\defeat A'_1$. There is an argument $A'_3\in\ext'-\ext$ such that $A'_3\defeat A_2$... and so on. We therefore construct a defeat chain of defeasible arguments
\begin{align}\label{eq:defeat_path}
\cdots\defeat A'_5\defeat A_4\defeat A'_3\defeat A_2\defeat A'_1,
\end{align}
where all primed arguments belong to $\ext'$ and all unprimed arguments are in $\ext$.

In general, suppose $A\defeat B$, then by Equation \ref{eq:ASPIC+_general_defeat}, $A\defeat C\subarg B$ for some $C$, and $A\not\prec_{SP}C$. By Corollary \ref{cor:larger_args_less_pref}, $B\precsim_{SP}C$. Assume for contradiction that $A\prec_{SP}B$, then by Lemma \ref{lem:SP_total_preorder_when_total}, $A\prec_{SP}C$ -- contradiction, so $A\not\prec_{SP}B$.

Equation \ref{eq:defeat_path} thus becomes $\cdots\not\prec_{SP} A'_5\not\prec_{SP} A_4\not\prec_{SP} A'_3\not\prec_{SP} A_2\not\prec_{SP} A'_1$. By Lemma \ref{lem:SP_total_preorder_when_total}, this is equivalent to $A'_1\precsim_{SP}A_2\precsim_{SP}A'_3\precsim_{SP}A_4\precsim_{SP}A'_5\precsim_{SP}\cdots$. By Equation \ref{eq:SP_arg_pref_non_strict} and Lemma \ref{lem:rules_in_stb_ext}, none of the adjacent arguments in this chain can have the same defeasible rules. This implies $A'_1\prec_{SP}A_2\prec_{SP}A'_3\prec_{SP}A_4\prec_{SP}A'_5\prec_{SP}\cdots$. The corresponding chain for defeasible rules is, by Equation \ref{eq:SP_arg_pref}, $DR\pair{A'_1}\ordneq_{SP}DR\pair{A_2}\ordneq_{SP}DR\pair{A'_3}\ordneq_{SP}DR\pair{A_4}\ordneq_{SP}\cdots$. As $\relsymb_d$ is a finite set, there are only finitely many possible sets of defeasible rules. This strictly ascending chain must therefore be finite, say of length $n$. Equation \ref{eq:defeat_path} must therefore be of finite length, terminating at an undefeated argument $B$, which may or may not be strict.
\begin{align*}
&B\defeat A'_{n-1}\defeat\cdots\defeat A_2\defeat A'_1\text{ or }B\defeat A_{n-1}\defeat\cdots\defeat A_2\defeat A'_1,
\end{align*}
for some $n\in\nat^+$. In the first case, $B\in\ext-\ext'$ is an undefeated argument, so $\ext'$ is not a stable extension -- contradiction. In the second case, by similar reasoning, $\ext$ is not a stable extension -- contradiction. There cannot be another stable extension $\ext'$ of $DG\pair{T^+}$, so $\ext$ is the unique stable extension of $DG\pair{T^+}$.
\end{proof}

\noindent The defeat graphs $DG\pair{T^+}$ of LPDTs $T^+$ thus have a unique stable extension.

\subsubsection{The Representation Theorem: Statement and Proof}

\noindent In this section we state and prove the representation theorem which relates the stable extension of $DG\pair{T^+}:=\ang{\alg,\defeat}$ with the extension of the corresponding LPDT $T^+:=\ang{D,\:W,\:<^+}$.

\begin{thm}\label{thm:rep_thm}
(The Representation Theorem) Let $AG\pair{T^+}$ be the attack graph corresponding to an LPDT $T^+$ with defeat graph $DG\pair{T^+}$ under $\precsim_{SP}$.
\begin{enumerate}
\item Let $E$ be the extension of $T^+$, which is unique (Section \ref{sec:rev_PDL}). Then there exists a unique stable extension $\ext\subseteq\alg$ of $DG\pair{T^+}$ such that $Conc\pair{\ext}=E$.
\item Let $\ext\subseteq\alg$ be the unique stable extension of $DG\pair{T^+}$ by Theorem \ref{thm:total_still_has_unique_stable_extension}, then $Conc(\ext)$ is the extension of $T^+$.
\end{enumerate}
\end{thm}
\begin{proof}
\textbf{Proof of part 1:} To prove the first statement we construct $\ext$ in terms of $E$ and show $\ext$ is a stable extension of $\ang{\alg,\defeat}$. By Theorem \ref{thm:total_still_has_unique_stable_extension}, this stable extension is unique. We finally show $Conc(\ext)=E$.

Given $E$, we let $\ext:=Args\pair{f\pair{NBD(E)}}\subseteq\alg$, where $NBD(E)$ is defined in Equation \ref{eq:lem_NBD} and $Args\pair{\:\cdot\:}$ is defined by Equation \ref{eq:def_Args(R)}. This set is unique from the properties of $Args$. For notational convenience we let $R:=f\pair{NBD(E)}$. We show $\ext$ is a stable extension.

Assume for contradiction that $\ext$ is not cf, which means there are arguments $A,B\in\ext$ such that $A\defeat B$, which means $A\attk B$ by Equation \ref{eq:ASPIC+_general_defeat}. Let $a:=Conc(A)$. As $A\in\ext$, Equation \ref{eq:def_Args(R)} and the definition of $\ext$ means that $DR(A)\subseteq R$. This means from $W$ and the defaults of $f^{-1}\pair{DR(A)}\subseteq D$ ($f$ is defined by Equation \ref{eq:bij_defaults_def_rules}), which are non-blocked defaults in $E$, it follows that $a\in E$. As $E$ is deductively closed, this means $\neg\neg a\in E$. Now let $B'\subarg B$ be the argument such that $TopRule(B')=\pair{b\Rightarrow \neg a}$ for some appropriate intermediate conclusion $b\in Conc\pair{Sub(B)}$. As $B\in\ext$, this means $\pair{b\Rightarrow \neg a}\in R$. By Equation \ref{eq:lem_NBD}, this means $\neg\neg a\notin E$ -- contradiction. Therefore, $\ext$ is cf.

To show $Args(R)$ defeats all other arguments, let $B\notin Args(R)$ be arbitrary. Let $r\in DR(B)-R$ be some rule. Let $B'\subarg B$ be such that $TopRule(B')=r$. The rule $r$ corresponds to a default $f^{-1}(r)=\frac{\theta:\phi}{\phi}\notin NBD(E)$ (Equation \ref{eq:bij_defaults_def_rules}). By Equation \ref{eq:lem_NBD}, we have two cases: either $\theta\notin E$ or $\neg\phi\in E$.

Case 1: If $\neg\phi\in E$, then we now show there exists an argument $A\in Args(R)$ such that $A\defeat B'$ and hence $A\defeat B$, under $\precsim_{SP}$. By Equations \ref{eq:ext_base} and \ref{eq:ext_ind}, there is some $i\in\nat$ such that $\neg\phi\in E_i$. Suppose $i=0$ then $W\models\neg\phi$. Compactness means there is some $W'\subseteq_\text{fin} W$ such that $W'\models\neg\phi$. We can construct an argument $A$ such that $Prem(A)=W'$ and $Conc(A)=\neg\phi$ as there will be appropriate combinations of strict rules in $\relsymb_s$ and premises in $\mathcal{K}_n$, so $A\rightharpoonup B$. As $DR(A)=\es\subseteq R$, we must have $A\in Args(R)$. As $A$ is strict, $A\hookrightarrow B$ is guaranteed by Corollary \ref{cor:empty_set_is_top} of $\precsim_{SP}$.

Now suppose that $i>0$, then $\neg\phi\in E_j$ where $j>0$ is the witness for $i$. Let $d_j\in D$ be the default that is $<^+$-greatest active in the layer $E_j$, so the set of defaults that are used in concluding $\neg\phi$ (up to the application of deductive rules) is $S:=\set{d_0,\:\ldots,\:d_{j-1}}\subseteq GD_{j-1}\pair{<^+}\subseteq NBD(E)$. We can construct an argument $A$ such that $Prem(A)\subseteq W$, $Conc(A)=\neg\phi$ and $DR(A)=f(S)$. Clearly, $DR(A)=f(S)\subseteq f\pair{NBD(E)}=:R$ and hence $A\in Args\pair{R}$. It is clear that $A\attk B$, so we need to show $A\not\prec_{SP} B$.

Given that $\neg\phi\in E_j$, it must be the case that $\phi\notin E_j$. Therefore, $f^{-1}\pair{r}$ is not $<^+$-greatest active for all extension layers $E_0,\:\ldots,\:E_{j-1}$. Suppose for contradiction that there is some rule $s\in DR(A)$ such that $s<_{SP} r$. Then by Equation \ref{eq:SP_ord_def}, $r$ must be $<^+$-greatest active at some $E_k$ for $k<j-1$, which would then result in $\phi$ in $E_{k+1}$, therefore preventing $\neg\phi\in E_j$ -- contradiction.\footnote{For this $s\in DR(A)$, the assumption that $s<_{SP}r$ means that the antecedent of $f^{-1}(r)$ is already in the appropriate extension layer.} Therefore, $r$ is $<_{SP}$-smaller than all rules in $DR(A)$. By Equation \ref{eq:SP_arg_pref}, we must have $B\prec_{SP} A$, and hence $A\not\prec_{SP} B$, so $A\defeat B$. Therefore, for the case of $\neg\phi\in E$, $Args(R)$ defeats all arguments outside it. Therefore, in all cases, if $\neg\phi\in E$, there is some argument $A\in\ext$ that defeats $B$.

Case 2: We will assume $\theta\notin E$, and show that it leads to a contradiction with the method of infinite descent.

If $\theta\notin E$, then there is some proper subargument $B''\subset_\text{arg} B'$ such that $Conc(B'')=\theta$. Since $\theta\notin E$ then it is the case that neither $DR\pair{B''}=\es$ nor in $Args\pair{f\pair{NBD\pair{E}}}$. This is because if $DR\pair{B''}=\es$ then $W\supseteq Prem(B'')\models\theta$ so $\theta\in E_0\subseteq E$ by Equation \ref{eq:ext_base}. Further, if $B''\in Args\pair{f\pair{NBD\pair{E}}}$ then $f^{-1}\pair{DR(B)}\subseteq NBD(E)$, so by Equation \ref{eq:lem_NBD}, we have $\theta\in E$.

Therefore, as $B''$ is neither strict nor in $Args\pair{f\pair{NBD\pair{E}}}$ there is some other defeasible rule $s\in DR\pair{B''}-f\pair{NBD\pair{E}}$. We can repeat the above reasoning with the rule $s$ instead of $r$: suppose $s=\frac{\theta':\phi'}{\phi'}$, then either $\theta'\notin E$ or $\neg\phi'\in E$. In the latter case we can construct an argument $A'\in\ext$ concluding $\neg\phi'$ which then defeats $B''$ as in the case when we assumed $\neg\phi\in E$. In the former case we repeat the reasoning in the previous paragraph, but we cannot do this indefinitely as arguments are well-founded. We will end up with either a proper subargument of $B''$ or an argument in $Args\pair{R}$ concluding $\theta$. In both cases $\theta\in E$ is true, so assuming $\theta\notin E$ will lead to contradiction by the method of infinite descent.

Therefore, the only reason for $r\notin R$ is because $\neg\phi\in E$. We have shown there is an argument $A$ that defeats any argument containing the rule $r$. As $r$ belongs to some arbitrary $B\notin\ext$, this means $\ext:=Args(R)$ defeats all arguments outside of it and hence it is a stable extension.

To show that $E=Conc\pair{\ext}$, we show $E\subseteq Conc\pair{\ext}$ and $Conc\pair{\ext}\subseteq E$. In the first case, let $\theta\in E$ so there is some $i\in\nat$ such that $\theta\in E_i$ by Equations \ref{eq:ext_base} and \ref{eq:ext_ind}.

If $\theta\in E_0$, we have $W\models\theta$ so by compactness there is some $\Delta\subseteq_\text{fin}W$ where $\Delta\models\theta$. Given that $\relsymb_s$ has all rules of proof we can construct a strict argument $A$ such that $Prem_n(A)=\Delta$ and $Conc(A)=\theta$. As strict arguments are undefeated, $A\in\ext$ so $\theta\in Conc\pair{\ext}$.

If $\theta\in E_k$ for some $k\in\nat^+$, we can construct a defeasible argument $A$ concluding $\theta$ such that $DR(A)\subseteq R$ and hence $A\in\ext$, so $\theta\in Conc\pair{\ext}$. Specifically, we construct an argument whose defeasible rules correspond to the defaults added to $E$ up to $E_k$.

Conversely, if $\theta\in Conc(\ext)$ there is an argument in $\ext$ concluding $\theta$. If this argument is strict then $\theta\in E_0\subseteq E$, else, as the defeasible rules are in $R$ then $\theta\in E_k\subseteq E$ for some $k\in\nat^+$ that indicates when all of the appropriate defaults needed to conclude $\theta$ are included. 

\textbf{Proof of part 2:} We show $Conc(\ext)\subseteq E$ and $E\subseteq Conc(\ext)$. For the former, if $\theta\in Conc(\ext)$ then there is some $A\in\ext$ concluding $\theta$. If $A$ is strict then $\theta\in E_0\subseteq E$. If $A$ is defeasible, then say $DR(A)=\set{r_i}_{i=1}^k$ for some $k\in\nat^+$. These defeasible rules do not introduce any inconsistency to $\ext$ by Lemma \ref{lem:acf_S_consistent}. Consider the set of corresponding defaults $\set{d_i}_{i=1}^k\subseteq D$ to $DR(A)$. We can choose the smallest index $j\in\nat$ such that all of the conclusions of these defaults are included in $E_j\subseteq E$. This is because either $\set{d_i}_{i=1}^k\subseteq GD(E)$, or there is some $d_i\in SAD(E)$, but that would mean $cons(d_i)\in E$ so for some $l\in\nat$, $cons(d_i)\in E_l$. Therefore, under deductive closure, $\theta\in E_{j+1}\subseteq E$, so $\theta\in E$. This shows that $Conc\pair{\ext}\subseteq E$.

Conversely, let $\theta\in E$, so there is some $i\in\nat$ such that $\theta\in E_i$. If $i=0$, then there is a strict argument $A$, necessarily in $\ext$ as it is undefeated, that concludes $\theta$ so $\theta\in Conc\pair{\ext}$. If $i>0$, then from $\theta\in E_i$ we can consider the defaults added to $E_i$, and use the corresponding defeasible rules to construct an argument $A$ such that $Prem(A)\subseteq W$, $Conc(A)=\theta$ and $DR(A)$ to contain exactly those defeasible rules. By definition of $<_{SP}$ and Algorithm \ref{alg:gen_stab_ext}, these defeasible rules are all present in $\ext$, so $DR(A)\subseteq DR\pair{\ext}$. By Lemma \ref{lem:rules_in_stb_ext}, $A\in\ext$. Therefore, $\theta\in Conc\pair{\ext}$. This shows that $E\subseteq Conc\pair{\ext}$.
\end{proof}

\noindent Theorem \ref{thm:rep_thm} means that PDL, where the default priority $<$ is a total order, is sound and complete with respect to its argumentation semantics; the inferences of PDL can be formally seen as the conclusions of justified arguments.

Given the definition of $\prec_{SP}$ and the translation of a LPDT to its defeat graph as described in Section \ref{sec:instantiation_choice_of_variables}, we can visualise the representation theorem in the following diagram:
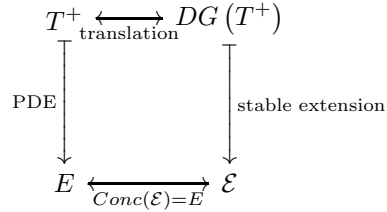
\[
\begin{diagram}
\node{T^+} \arrow{e,b,<>}{\text{translation}} \arrow{s,l,T}{\text{PDE}}
\node{DG\pair{T^+}} \arrow{s,r,T}{\text{stable extension}}\\
\node{E} \arrow{e,b,<>}{Conc\pair{\ext}=E}
\node{\ext}
\end{diagram}
\]

\subsection{Satisfaction of Rationality Postulates}\label{sec:normative_rationality_of_PDL_inst}

In this section, we prove \textit{directly} that our instantiation of ASPIC$^+$ to PDL satisfies the Caminada-Amgoud rationality postulates and hence is normatively rational. We do this by investigating some properties of the stable extension of this defeat graph. Notice that we do not appeal to the sufficient conditions articulated by ASPIC$^+$ that, if satisfied, will guarantee normative rationality. We will discuss why in Section \ref{sec:discussion_conclusions}.

\subsubsection{The Stable Extension is Grounded}

\begin{lem}\label{lem:def^n_contains_stable_ext}
Let $T^+$ be an LPDT with attack graph $AG\pair{T^+}$ and defeat graph $DG\pair{T^+}$. Let $\chi:\pow\pair{\alg}\to\pow\pair{\alg}$ be the characteristic function. Let $\ext\subseteq\alg$ be the stable extension of $DG\pair{T^+}$. Then $\pair{\exists n\in\nat^+}\ext\subseteq\chi^n\pair{\es}$, where $\chi^n$ denotes the $n^\text{th}$ iterate of $\chi$.
\end{lem}

Before we prove Lemma \ref{lem:def^n_contains_stable_ext}, we first establish some notation that will be used in the proof. Recall \textit{interval notation}: for any toset $\ang{T,<}$ and $a,b\in T$, define the subsets:
\begin{align*}
(a,b)&:=\set{x\in T\:\vline\: a < x < b},\:[a,b]:=\set{x\in T\:\vline\:a\leq x\leq b},\\
[a,b)&:=\set{x\in T\:\vline\:a\leq x<b}\text{ and }(a,b]:=\set{x\in T\:\vline\:a< x\leq b}.
\end{align*}
Recall that if $b\leq a$ then $(a,b)=[a,b)=(a,b]=\es$, and if $b<a$ then $[a,b]=\es$. As $D$ is finite let $N:=\abs{\relsymb_d}$. Given the LPDT $T^+$ with defeat graph $DG\pair{T^+}$ and stable extension $\ext$, define the set, for $0\leq k\leq N$,
\begin{align}\label{eq:blocked_defeasible_rules}
\relsymb_d-DR\pair{\ext}=:\set{r_1,\ldots, r_k},
\end{align}
where $k=0$ means $DR\pair{\ext}=\relsymb_d$, and $k=N$ means $DR\pair{\ext}=\es$. WLOG, we arrange the indices for these rules such that $r_{l+1}<_{SP} r_l$, for $1\leq l\leq k-1$. This is the set containing the rules that do not feature in $\ext$.

Given the strict toset $\ang{\relsymb_d,<_{SP}}$, denote $r_{\max}:=\max_{<_{SP}}\relsymb_d$ and similarly for $r_{\min}$. Both $r_{\max}$ and $r_{\min}$ are uniquely defined as $<_{SP}$ is total. To isolate the defeasible rules that do not make up the arguments in $\ext$, we partition $\relsymb_d$ from smallest to largest in $<_{SP}$ as follows:
\begin{align}\label{eq:partition}
[r_{\min},r_k),\set{r_k},\pair{r_k,r_{k-1}},\set{r_{k-1}},\ldots,\set{r_2},\pair{r_2,r_1},\set{r_1},(r_1,r_{\max}].
\end{align}
This places the defeasible rules that do not feature in $\ext$ into their own singleton sets along the chain $<_{SP}$. These singleton sets contain precisely the defeasible rules skipped over by Algorithm \ref{alg:gen_stab_ext} when constructing $\ext$. Note the first and last of these sets may be empty, e.g. when $r_k=r_{\min}$. We name these sets: for $1\leq i\leq k$, $P_i:=\set{r_i}$. Similarly, for $2\leq i\leq k$, $I_i:=\pair{r_i,r_{i-1}}\subseteq\relsymb_d$. We also define $I_1:=(r_1,r_{\max}]\subseteq\relsymb_d$ and $I_{k+1}:=[r_{\min},r_k)\subseteq\relsymb_d$. Equation \ref{eq:partition} can be written as:
\begin{align}\label{eq:partition_named_sets}
I_{k+1},\:P_k,\:I_k,\:P_{k-1},\:\ldots,P_2,\:I_2,\:P_1,\:I_1.
\end{align}
Notice from Equation \ref{eq:partition} that $\bigcup_{j=1}^{k+1}I_j=\relsymb_d-\set{r_k,r_{k-1},\ldots, r_2,r_1}=DR\pair{\ext}$.

We define the following counterpart sets of arguments to those in Equation \ref{eq:partition_named_sets}. For $1\leq i\leq k$, define
\begin{align}\label{eq:APi}
AP_i:=\set{A\in\alg\:\vline\:DR(A)\cap P_i\neq\es}\subseteq\alg.
\end{align}
These are the sets of defeasible arguments that have at least one rule excluded from $\ext$. By Lemma \ref{lem:rules_in_stb_ext} and the definition of the $P_i$ sets, it is easily shown that for all $1\leq i\leq k$, $AP_i\cap\ext=\es$. Further, for $1\leq i\leq k+1$,
\begin{align}\label{eq:AIi_def}
AI_i:=\set{A\in\alg\:\vline\:DR(A)\cap I_i\neq\es\text{ and }DR(A)\subseteq\bigcup_{j=1}^i I_j}.
\end{align}
These are the sets of defeasible arguments where the arguments only have rules from these intervals (as subsets of $DR\pair{\ext}$), with at least one such rule from the $<_{SP}$-lowest ranked interval $I_i$. We also define
\begin{align}\label{eq:strict_arguments_alternative_notation}
AI_0:=\set{A\in\alg\:\vline\:DR(A)=\es}=Args\pair{\es},
\end{align}
which by Equation \ref{eq:def_Args(R)} is the set of all strict arguments. Clearly $AI_0\subseteq\chi\pair{\es}$.


\begin{lem}\label{lem:strong_induction_base_case}
It is the case that $AI_1\subseteq\chi\pair{\es}$.
\end{lem}
\begin{proof}
The set $AI_1$ is cf (conflict free): assume for contradiction that $A,B\in AI_1$ such that $A\defeat B$. Then as $DR(A),\:DR(B)\subseteq I_1\subseteq DR\pair{\ext}$, by Equations \ref{eq:blocked_defeasible_rules} and \ref{eq:partition} and Lemma \ref{lem:rules_in_stb_ext}, $A,B\in\ext$ -- contradiction, because $\ext$ is cf. Now assume for contradiction that $A\in AI_1$ is defeated, so there is some $B\in\alg$ such that $B\defeat A$. Clearly if $B\in\ext$, then $\ext$ is not cf because $A\in\ext$. Therefore, $B\notin\ext$. By Lemma \ref{lem:rules_in_stb_ext}, $DR(B)\not\subseteq DR\pair{\ext}$ so for some $1\leq i\leq k$, $DR(B)\cap P_i\neq\es$. However, by Equation \ref{eq:partition}, there is a rule in $DR(B)-DR(A)$, namely $r_i$, such that it is $<_{SP}$-smaller than all rules in $DR(A)$ (and hence $<_{SP}$-smaller than all rules in $DR(A)-DR(B)$). Therefore, $B\prec_{SP}A$, and hence $B$ cannot defeat $A$. Therefore, all arguments in $AI_1$ are also undefeated, so $AI_1\subseteq\chi\pair{\es}$.
\end{proof}

For the purposes of the proof of Lemma \ref{lem:def^n_contains_stable_ext}, we define, for $i>k+1$,
\begin{align}\label{eq:AI_set_stabilise}
AI_i=AI_{k+1}.
\end{align}
Note that the sets of arguments in Equation \ref{eq:partition_named_sets} do not partition $\alg$, as an argument can conceivably have rules from two or more of the $I_j$ sets. We now apply these ideas to prove Lemma \ref{lem:def^n_contains_stable_ext}.

\begin{proof}
(Proof of Lemma \ref{lem:def^n_contains_stable_ext}) Given our setup, let $\ext$ be the stable extension of $DG\pair{T^+}$. By Theorem \ref{thm:total_still_has_unique_stable_extension}, $\ext$ is unique and could only have been constructed by Algorithm \ref{alg:gen_stab_ext}. Algorithm \ref{alg:gen_stab_ext} begins with $Args\pair{\es}$, then adds rules in $\relsymb_d$ from $<_{SP}$-largest to $<_{SP}$-smallest as long as the resulting set with the rule contains no arguments attacking each other. From the above notation, it is exactly the rules in the sets $P_i$ that, when included, create arguments that attack each other. This is why these rules in $P_i$ do not feature in $\ext$.

We use strong induction to show that for $i\in\nat^+$, $AI_i\subseteq\chi^i\pair{\es}$. The base case, $i=1$, follows from Lemma \ref{lem:strong_induction_base_case}. For the strong inductive step, assume $AI_j\subseteq\chi^j\pair{\es}$ for all $1\leq j\leq i$. We will show that $AI_{i+1}\subseteq\chi^{i+1}\pair{\es}$. Let $A\in AI_{i+1}$ be arbitrary. This means $DR(A)\subseteq\bigcup_{j=1}^{i+1}I_j$ and $DR(A)\cap I_{i+1}\neq\es$ by Equation \ref{eq:AIi_def}. Either $A$ is defeated by an argument in $\bigcup_{j=1}^i AP_j$ or it is not, where $AP_j$ is defined in Equation \ref{eq:APi}. If $A$ is not defeated by an argument in $\bigcup_{j=1}^i AP_j$, then $A\in\chi\pair{\es}$ as it is undefeated; $A$ cannot be defeated by some argument $B$ in $AP_j$, for $j>i$, because in that case $B\prec_{SP}A$. As $\chi$ is $\subseteq$-monotonic, $A\in\chi^{i+1}\pair{\es}$. Otherwise, if $\bigcup_{j=1}^i AP_j\defeat A$, then there is some $1\leq j\leq i$ such that $AP_j\defeat A$. Call the witness to $j$ $j_0$, so $AP_{j_0}\defeat A$. Say the defeating argument is $C\in AP_{j_0}$. But by definition of $AP_{j_0}$,
\begin{align*}
&\pair{\exists B\in\bigcup_{s=0}^{j_0}AI_s}B\defeat C\text{, $s=0$ is included as $B$ may be strict,}\\
\Leftrightarrow&\pair{\exists 0\leq s\leq j_0}\pair{\exists B\in AI_s}B\defeat C\\
\Rightarrow&\pair{\exists 0\leq s\leq j_0}\pair{\exists B\in \chi^s\pair{\es}}B\defeat C\text{ by our strong inductive hypothesis.}
\end{align*}
This means $A$ is defended by $\chi^s\pair{\es}$, so $A\in \chi^{s+1}\pair{\es}$, for some $0\leq s\leq j_0$. As $s\leq j_0\leq i$, this means $s\leq i$ and hence $s+1\leq i+1$. By $\subseteq$-monotonicity of $\chi$, $A\in\chi^{i+1}\pair{\es}$. This establishes the inductive step.

However, this proof by induction proves this for all $i\in\nat^+$. What happens when $i>k+1$? If $i>k+1$, then by Equation \ref{eq:AI_set_stabilise}, $AI_i=AI_{k+1}\subseteq\chi^{k+1}\pair{\es}\subseteq\chi^{i}\pair{\es}$ and we have no more defeasible rules to add. As the sequence $AI_i$ stabilises the result holds for all $i\in\nat^+$ trivially.

Now, as $\es,\:\chi\pair{\es},\:\chi^2\pair{\es}\ldots$ form an $\subseteq$-increasing sequence in $\pow\pair{\alg}$, we can take the union of the equations $AI_i\subseteq\chi^i\pair{\es}$ and invoke monotonicity of $\chi$:\footnote{Strictly speaking the union should be over all $i\in\nat$ but because the $AI_i$ sequence stabilises we only have to care about $0\leq i\leq k+1$.}
\begin{align}
\bigcup_{i=1}^{k+1}AI_i\subseteq\bigcup_{i=1}^{k+1}\chi^i\pair{\es}=\chi^{k+1}\pair{\es}.\nonumber
\end{align}
We then take the union of both sides with the set of all strict arguments. The left hand side becomes $\ext$. This is because the union of the $AI_j$ sets from $j=0$ to $k+1$ means any argument in that set cannot have any rules in the $P_k$ sets, and therefore $DR(A)\subseteq DR\pair{\ext}$ and hence $A\in\ext$ by Lemma \ref{lem:rules_in_stb_ext}.

As the set of all strict arguments is contained in $\chi\pair{\es}$ because they are undefeated, the right hand side stays the same. Therefore, we obtain $\ext\subseteq\chi^{k+1}\pair{\es}$, where $0\leq k\leq N\in\nat$ is the number of defeasible rules blocked from $\ext$, which is a natural number.\footnote{Notice if $k=0$, there is no conflict, all arguments in $\ext$ are undefeated, so $\ext\subseteq\chi\pair{\es}$.} This shows the result.
\end{proof}

\begin{lem}\label{lem:grounded_contains_def^n_empty}
Let $\ang{A,\to}$ be an abstract argumentation framework and $\chi$ its characteristic function. Let $G\subseteq A$ be the grounded extension. Then $\pair{\forall n\in\nat}\chi^n\pair{\es}\subseteq G$.
\end{lem}
\begin{proof}
Immediate by induction on $n$: $\chi$ is $\subseteq$-monotonic and $G$ is complete.
\end{proof}

\noindent We now instantiate the abstract framework $\ang{A,\attk}$ in Lemma \ref{lem:grounded_contains_def^n_empty} to the defeat graph $\ang{\alg,\defeat}$ of a LPDT.

\begin{cor}\label{cor:grounded_is_stable}
Suppose we have an LPDT $T^+$ with attack graph $AG(T^+):=\ang{\alg,\attk,\precsim_{SP}}$ and defeat graph $DG(T^+):=\ang{\alg,\defeat}$. The characteristic function $\chi$ is as usual. The stable extension $\ext\subseteq\alg$ of $DG(T^+)$ is grounded.
\end{cor}
\begin{proof}
From Lemma \ref{lem:def^n_contains_stable_ext}, there exists some $n\in\nat$ such that $\ext\subseteq\chi^n\pair{\es}$. From Lemma \ref{lem:grounded_contains_def^n_empty}, we have $\chi^n\pair{\es}\subseteq G$, where $G\subseteq\alg$ is the grounded extension. But by definition, $G\subseteq\ext$ because stable extensions are complete. Therefore, we have, for this $n\in\nat$, $G\subseteq\ext\subseteq\chi^n\pair{\es}\subseteq G$, so $\ext=G$.
\end{proof}

\subsubsection{The Trivialisation and Rationality Theorems}

The trivialisation theorem states that if the underlying default priority is total, then all of Dung's argumentation semantics are equivalent.

\begin{thm}\label{thm:trivialisation_theorem}
(The Trivialisation Theorem) The defeat graph $\ang{\alg,\defeat}$ of an LPDT $T^+$ has a unique complete extension that is grounded, preferred and stable.
\end{thm}
\begin{proof}
Let $C$ be any complete extension of $\ang{\alg,\defeat}$, which means $C$ is cf and $\chi\pair{C}=C$. Let $G$ be the grounded extension, then $G\subseteq C$ by definition. As the (unique) stable extension $\ext$ is grounded by Corollary \ref{cor:grounded_is_stable}, we have $G=\ext$, therefore $\ext\subseteq C$. This means either $\ext=C$ or $\ext\subset C$. In the latter case, there will be some $B\notin\ext$ such that $B\in C$, but as $\ext\subset C$ is stable, we must have some $A\in\ext$ (so $A\in C$) such that $A\defeat B$. Therefore, $C$ is not cf -- contradiction. Therefore, $C=\ext$. As $\ext$ is unique, $C$ is unique. Therefore, $\ang{\alg,\defeat}$ has a unique complete extension that is grounded, stable and hence preferred.
\end{proof}

We now prove that this instantiation of ASPIC$^+$ to PDL satisfies the requirements for normative rationality \cite{Caminada:07}. Recall that when instantiated to FOL, $Cl_{\relsymb_s}$ becomes deductive closure.

\begin{thm}\label{thm:PDL_is_rational_direct}
(The Rationality Theorem) Let $\ang{\alg,\:\attk,\:\precsim_{SP}}$ be the ASPIC$^+$ attack graph of PDL and let $\ext$ be any of the complete extensions of the corresponding defeat graph $\ang{\alg,\defeat}$. Our instantiation satisfies the Caminada-Amgoud rationality postulates.
\end{thm}
\begin{proof}
By Theorem \ref{thm:total_still_has_unique_stable_extension}, $\ang{\alg,\defeat}$ has a unique stable extension $\ext$, which is a complete and an admissible extension. It is sufficient to prove the postulates for $\ext$ because by Theorem \ref{thm:trivialisation_theorem}, $\ang{\alg,\defeat}$ only has $\ext$ as its sole complete extension.
\begin{enumerate}
\item To show that $\ext$ is subargument closed, recall that Algorithm \ref{alg:gen_stab_ext} gives an explicit construction of $\ext$, which is of the form $Args(R)$ for some $R\subseteq\relsymb_d$ which is subargument closed (Equation \ref{eq:def_Args(R)}).
\item The representation theorem states that $Conc\pair{\ext}=E$ and as $E$ is deductively closed, $Conc\pair{\ext}$ is closed under strict rules.
\item As $W$ is consistent and $Conc\pair{\ext}$ is the extension, $Conc\pair{\ext}$ must also be consistent and its deductive closure is consistent.
\end{enumerate}
This shows the result.
\end{proof}

\noindent The rationality theorem establishes that this instantiation of ASPIC$^+$ to PDL satisfies all of the Caminada-Amgoud rationality postulates and is a normatively rational instantiation of ASPIC$^+$.

Finally, the consistency of $\ext$ on the side of PDL allows us to establish a stronger notion of cf for $\ext$ on the side of argumentation. This is already implicit in Algorithm \ref{alg:gen_stab_ext} Line \ref{alg_line:cond_start}.

\begin{cor}\label{cor:stb_ext_acf}
Let $\ang{D,W,<^+}$ be an LPDT with attack graph $\ang{\alg,\attk,\precsim_{SP}}$ and corresponding defeat graph $\ang{\alg,\defeat}$ that has a unique stable extension $\ext$. We have $\ext^2\cap\attk=\es$, i.e. no two arguments in $\ext$ attack each other.
\end{cor}
\begin{proof}
Given the hypotheses, assume for contradiction that $A\attk B$ for $A,B\in\ext$. WLOG we can assume that $Conc(A)=\theta$ and $Conc(B)=\neg\theta$ with $TopRule(B)\in\relsymb_d$, by Equation \ref{eq:attack} and that $\ext$ is subargument closed. Therefore, $Conc\pair{\ext}$ is inconsistent, because $\theta,\neg\theta\in Conc\pair{\ext}$. This violates the Rationality Theorem -- contradiction. Therefore, no two arguments in $\ext$ attack each other.
\end{proof}

\subsubsection{Inconsistent Arguments}\label{sec:inconsistent_args}

We have stated in Section \ref{sec:rev_ASPIC+} that arguments are constrcted freely from the premises and rules. In this instantiation, it is possible to construct arguments that are inconsistent, either in their intermediate conclusions or their conclusion.

\begin{eg}
Consider the rules $\pair{\Rightarrow a}$ and $\pair{\Rightarrow\neg a}$ and arguments $A=\sqbra{\Rightarrow a}$, $B=\sqbra{\Rightarrow\neg a}$ and $C=\sqbra{A,B\to\bot}$. Then $C$ is inconsistent in its conclusion. Further, for any $c\in\LSent$ the argument $A^+:=\sqbra{A\to\pair{a\vee c}}$, so given that our strict rules are the rules of proof in FOL, we can construct the argument $D=\sqbra{A^+,B\to c}$ for any $c$. The intermediate conclusions of $D$ are inconsistent.
\end{eg}

We call an argument \textit{inconsistent} iff $Conc\pair{Sub(A)}\subseteq\LForm$ is an inconsistent set in FOL. It is possible to construct such arguments in $\alg$. We can ignore these arguments by focussing on $\alg-\set{A\in\alg\:\vline\:A\text{ is inconsistent}}$ and restricting $\attk$ in the usual way, but this seems inelegant especally when the process of argumentation is meant to resolve inconsistencies. If we do include inconsistent arguments, the very least is that they should not be justified. By Theorems \ref{thm:total_still_has_unique_stable_extension} and \ref{thm:trivialisation_theorem}, there is only one way of justifying arguments: an argument $A$ is \textit{justified} iff $A\in\ext$. Lemma \ref{lem:acf_S_consistent} and Corollary \ref{cor:stb_ext_acf} ensure that if $A$ is inconsistent, then $A\notin\ext$.

In summary, although it is possible to have inconsistent arguments in $\alg$, they can never be justified and we do not need to be concerned with them.

\subsection{Summary}

In this section, we have provided an argumentative characterisation of PDL inference that is sound, complete and normatively rational, in the case where our default priority is a strict total order. We can construct ASPIC$^+$ arguments and attacks (Section \ref{sec:instantiation_choice_of_variables}). The structure-preference relation, $\precsim_{SP}$, takes into account both the default priority $<$ and the logical structure of arguments. This is motivated by how PDL adds defaults when constructing extensions (Section \ref{sec:pref}). The representation theorem states that under $\precsim_{SP}$, the PDL extension and the conclusion set of the stable extension correspond exactly (Section \ref{sec:rep_thm}, Theorem \ref{thm:rep_thm}). We can prove directly that the stable extension of interest satisfies the Caminada-Amgoud rationality postulates (Section \ref{sec:normative_rationality_of_PDL_inst}, Theorem \ref{thm:PDL_is_rational_direct}). As this is the only complete extension of our defeat graphs, our instantiation satisfies the postulates. Finally, we do not need to explicitly prevent the construction of inconsistent arguments, because they are never justified (Section \ref{sec:inconsistent_args}).

\section{On Lifting the Assumption of a Total Order Default Priority}\label{sec:lift_total_assumption}



In Section \ref{sec:ASPIC+_to_PDL}, we have provided an argumentative characterisation of PDL inference where the default priority $<$ is a strict total order. It seems that we have lost generality but this is not the case because calculating an extension in PDL always presupposes a linearisation $<^+$ of $<$ (Section \ref{sec:rev_PDL}), and Theorem \ref{thm:rep_thm} shows that for \emph{any} such linearisation the correspondence of inferences between PDL and its argumentation semantics holds.

But argumentation can also define argument preference relations based on an underlying \textit{partial} order. We now investigate how to lift the assumption that $<$ is total for the LPDT $T^+$, such that the resulting multiple stable extensions each correspond to an extension of the underlying PDT $T$. Our underlying representation of PDL in ASPIC$^+$ is the same as in Section \ref{sec:instantiation_choice_of_variables}, but now $<_D$ is a strict \textit{partial} order.

\subsection{The Argument Preference Relation based on Partial Order Default Priorities}\label{sec:partial_order_pref_part2}


In Section \ref{sec:pref}, we devised the structure-preference (SP) argument preference relation $\precsim_{SP}$ which captures the PDL idea of adding the ``$<^+$-greatest active'' default (Equation \ref{eq:ext_ind}). If we translate a PDT $T$ directly into an argument graph $AG\pair{T}$ without first linearising $<$, the generalised version of $\precsim_{SP}$ should take into account the incomparabilities of rules while still respecting their logical structure. We formalise this idea by defining a string representation of the rules that will be used in an algorithm to calculate $<_D\:\mapsto\:<_{SP}$ for partial $<_D$. We will give a variation of the Penguin Triangle (Example \ref{eg:penguin_triangle}) as a running example.

\subsubsection{A Representation of Rules and their Ordering using Strings}\label{sec:string_representation}


Let $rulenames$ be a set of characters, with as many characters as there are rules in $\relsymb_d$. Let $g:\relsymb_d\to rulenames$ be a bijection such that each $r\in\relsymb_d$ has a single-character name\footnote{By ``name'' we do not mean the naming function $n:\relsymb_d\to\lang$ in Section \ref{sec:rev_ASPIC+} (page \pageref{sec:rev_ASPIC+}), which is still undefined ($n\equiv *$) in this case, but just what we label the rules with, e.g. the defeasible rule $r_7=(a\Rightarrow b)$ has the label or name $r_7$.} $g(r)$. Let $\star$ denote the Kleene star and $*$ denote string concatenation,\footnote{This is abuse of notation as we had earlier stated $*$ refers to undefined quantities (Section \ref{sec:notation}). But there are few undefined quantities and the meaning of $*$ will be clear from context.} and $len:rulenames^\star\to\nat$ returns the number of letters of the string. We will also assume that in each string $\sigma\in rulenames^\star$ there is an index assocated with each letter starting from $0$ and ending in $len\pair{\sigma}-1$. To iterate over the letters $l$ of the string $\sigma$ we will write $l\in\sigma$, which starts from the letter at index 0 and terminates at the letter at index $len\pair{\sigma}-1$. The \textit{empty string} is $\infl:=\text{``''}$ with $len\pair{\infl}=0$. We may put quotation marks around strings to emphasise that they are strings.

For $R\subseteq\relsymb_d$ such that $R:=\set{r_1,\ldots,r_k}$, we can form the string $g(r_1)*g(r_2)*\cdots*g(r_k)$, written $g(r_1)g(r_2)\cdots g(r_k)$. Notice that forming a string from a set imposes an order on the elements.

\begin{eg}\label{eg:ad_nauseam_encoding_strings}
(Example \ref{eg:not_disj_eli_WLP} continued) Suppose we have $r_i:=f\pair{d_i}$ for $f$ as in Equation \ref{eq:bij_defaults_def_rules}. Suppose $g\pair{r_1}=\text{``}a\text{''},\:g\pair{r_2}=\text{``}b\text{''},\:g\pair{r_3}=\text{``}c\text{''},\:g\pair{r_4}=\text{``}d\text{''}\text{ and }g\pair{r_5}=\text{``}e\text{''}$. Then for the set $S=\set{r_1,r_2,r_5}$ we can form the strings (e.g.) $\text{``}abe\text{''}$ or $\text{``}bea\text{''}$, depending on which order we choose the rules to be in.
\end{eg}

\noindent As $g$ is a bijection we can define the reverse process. Suppose we have a string $\sigma$. We define the set of rules that are encoded by the letters of $\sigma$ as follows:
\begin{align}\label{eq:string_to_set}
\bigcup_{l\in\sigma}\set{g^{-1}\pair{l}}.
\end{align}
\noindent Notice that we lose the information about the index, but we will see that it does not matter. Notice also that if $\sigma=\infl$ then we have the empty union so the set of rules encoded by $\infl$ is $\es$.

\begin{eg}
(Example \ref{eg:ad_nauseam_encoding_strings} continued) Suppose we want to find the set of the string ``$ace$''. Applying Equation \ref{eq:string_to_set}, we get the set $\set{r_1,r_3,r_5}$.
\end{eg}

Lastly, we can transform strings into total orders: for $\sigma=\sigma_1\sigma_2\ldots\sigma_n$, where for $1\leq i\leq n$ the letter $\sigma_i$ has index $i-1$.\footnote{Here, $\sigma_i$ denotes one letter; the subscript $i$ is not a separate letter to $\sigma$ itself.} We can transform  $\sigma$ to the set
\begin{align}\label{eq:string_to_total_order}
\set{\pair{\sigma_1,\sigma_2},\pair{\sigma_1,\sigma_3},\ldots,\pair{\sigma_1,\sigma_n},\pair{\sigma_2,\sigma_3}\ldots,\pair{\sigma_{n-1},\sigma_n}}
\end{align}
using two nested loops ranging over the letters of $\sigma$ such that the pair of letters $(\sigma_i,\sigma_j)$ is added to the set iff $i<j$. Intuitively, given a selected letter $\sigma_i$ of a string, letters to the right of $\sigma_i$ are larger than $\sigma_i$, and letters to the left of $\sigma_i$ are smaller than $\sigma_i$.

\subsubsection{Algorithm and Example Calculation}\label{sec:alg_defined_for_generalised_SP}

We want to generalise the mapping $<_D\:\mapsto\:<_{SP}$, defined in Section \ref{sec:pref} for the case where $<_D$ is total, to arbitrary partial orders $<_D$. Furthermore, we want to simultaneously capture all possible linearisations of $<_D$. We now present the algorithm that calculates the generalised mapping $<_D\:\mapsto\:<_{SP}$ in two parts.

The first stage of the algorithm is a non-recursive depth first search algorithm that returns the set of all strings representing the rules chosen in accordance with both the preference and the structure as described in Section \ref{sec:string_representation}. This is articulated in Algorithm \ref{alg:part1}, which defines the function $\textsc{StructurePreference1}$. This function takes $\ang{\relsymb_d,<_D}$ as input and returns this set of strings.

The second stage of the algorithm is to turn the output of Algorithm \ref{alg:part1} into $<_{SP}$. This is done by translating each string in the output of Algorithm \ref{alg:part1} into a strict total order on $\relsymb_d$, and then taking their intersection. This is articulated in Algorithm \ref{alg:part2}, which defines a function $\textsc{StructurePreference2}$, which takes as input a set of strings, and calculates $<_{SP}$.

\begin{algorithm}[h]
\begin{algorithmic}[1]
\Function{StructurePreference1}{$\ang{\relsymb_d,\:<_D}$}
  \State $S_0\gets\set{\infl}$\label{alg_line:base}
  \State $N\gets\abs{\relsymb_d}$\label{alg_line:def_N}
  \For{$i=0$ to $N$}\label{alg_line:iterate_0_to_N}
    \State $S_{i+1}\gets\es$
    \For{$\lambda\in S_i$}
      \State $T^\lambda\gets\bigcup_{l\in\lambda}\set{g^{-1}\pair{l}}$\label{alg_line:previous_iteration}
      \State $R^\lambda\gets\max_{<_D}\sqbra{\set{s\in\relsymb_d\:\vline\:Ante(s)\subseteq Conc\pair{Args\pair{T^\lambda}}}-T^\lambda}$\label{alg_line:input_here}
      \If{$R^\lambda=\es$}
        \Return $S_i$
      \EndIf
      \For{$t\in g\pair{R^\lambda}$}\label{alg_line:loop_over_choices}
        \State $S_{i+1}\gets S_{i+1}\cup\set{t*\lambda}$\label{alg_line:concatenate}
      \EndFor
    \EndFor
  \EndFor
\EndFunction
\end{algorithmic}
\caption{Calculating $<_{SP}$ from $<_D$ on $\relsymb_d$, Part 1 -- generate a set of strings, each string is an order of the choice of rules from least preferred (the first letter, on the left) to the most preferred (the last letter, on the right), which essentially corrects $<_D$ for the argument structure and then linearises. Recall that in Line \ref{alg_line:concatenate}, $*$ refers to string concatenation. As $\relsymb_d$ is finite, Algorithm \ref{alg:part1} terminates. Throughout, $S_i$ is a set of strings, while $R^\lambda,\:T^\lambda\subseteq\relsymb_d$.}
\label{alg:part1}
\end{algorithm}

The intuition of Algorithm \ref{alg:part1} is as follows: we initialise the algorithm (Line \ref{alg_line:base}) and iterate $N+1$ times (Line \ref{alg_line:iterate_0_to_N}), where $N=\abs{\relsymb_d}$ (Line \ref{alg_line:def_N}). At each iteration Algorithm \ref{alg:part1} chooses all of the most preferred applicable rules at that stage. Each choice may render more rules active for the next iterations. It iterates over all such possibilities and repeats this process until the $N^\text{th}$ iteration, where there are no more rules to be chosen and the algorithm terminates. The result is a set of strings, which are read from right to left, where the right-most letter is the first choice of most preferred applicable defeasible rule, and the left-most letter is the last choice, which usually corresponds to a blocked default.

\begin{algorithm}[ht]
\begin{algorithmic}[1]
\Function{StructurePreference2}{$S$}
  \State $orders\gets\es$
  \For{$\lambda\in S$}
    \State $<_{SP}^+\gets\es$\label{alg_line:start_forming_chain}
    \For{$r\in\lambda$}
      \For{$s\in\lambda$}
        \If{$index(r)<index(s)$}
          \State $<_{SP}^+\gets<_{SP}^+\cup\set{(r,s)}$\label{alg_line:end_forming_chain}
        \EndIf
      \EndFor
    \EndFor
    \State $orders\gets orders\cup\set{<_{SP}^+}$
  \EndFor
  \Return $\bigcap orders$
\EndFunction
\end{algorithmic}
\caption{Calculating $<_{SP}$ from $<_D$ on $\relsymb_d$, Part 2 -- from the set of strings generated from Algorithm \ref{alg:part1}, we turn each string into a strict total order on $\relsymb_d$, and then take their intersection to return $<_{SP}$.}
\label{alg:part2}
\end{algorithm}

The intuition of Algorithm \ref{alg:part2} is that upon input of this set of strings, the algorithm turns each string into a strict total order over $\relsymb_d$ as described in the end of Section \ref{sec:string_representation} (Equation \ref{eq:string_to_total_order}), and then takes the intersection of all such orders to return $<_{SP}$. The intersection returns the ``core'' strict partial order which is the ``smallest'' change to the original $<_D$ that is compatible with all argument structures. Given $\ang{\relsymb_d,<_D}$, we define:
\begin{align}\label{eq:order_mapping_new_SP}
<_{SP}:=&\textsc{StructurePreference2(StructurePreference1$\sqbra{\ang{\relsymb_d,<_D}}$)}\nonumber\\
:=& F\pair{<_D}.
\end{align}

\noindent This is our method for calculating $\ang{\relsymb_d,<_{SP}}$ from $\ang{\relsymb_d,<_D}$ where $<_D$ is partial.

\begin{eg}\label{eg:penguin_triangle}
(Modified Penguin Triangle) Let $W=\es$,
\begin{align}
D=\set{d_1:=\frac{:a}{a},d_2:=\frac{a:b}{b},d_3:=\frac{:\neg b}{\neg b}}\nonumber
\end{align}
and consider the default priority $<=\set{\pair{d_3,d_2}}$.\footnote{Notice this is not the ``usual'' partial order priority that respects the specificity principle.} There are three possible linearisations of $<$ giving two possible extensions:
\begin{align*}
E_1:=&Th\pair{\set{a,b}}\text{ from }d_3<^+d_2<^+d_1\text{ and }d_3<^+d_1<^+d_2\text{, and}\\
E_2:=&Th\pair{\set{a,\neg b}}\text{ from }d_1<^+ d_3<^+ d_2.
\end{align*}
Let $r_i:=f\pair{d_i}$ for $i=1,2,3$ (Equation \ref{eq:bij_defaults_def_rules}). We illustrate these arguments in Figure \ref{figure:penguin}.


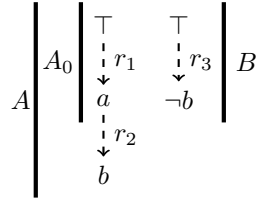
\begin{figure}[h]
\begin{center}
\begin{tikzpicture}[scale = 1]
\node (T1) at (0,0) {$ \top $};
\node (a) at (0,-1) {$ a $};
\draw [thick, dashed, ->] (T1) -- (a);
\node (r1) at (0.3,-0.5) {$r_1$};
\node (b) at (0,-2) {$ b $};
\draw [thick, dashed, ->] (a) -- (b);
\node (r2) at (0.3,-1.5) {$r_2$};
\draw [ultra thick] (-0.3,0.3)--(-0.3,-1.3);
\node (A0) at (-0.6,-0.5) {$A_0$};
\draw [ultra thick] (-0.9,0.3)--(-0.9,-2.3);
\node (A) at (-1.1,-1) {$A$};
\node (T2) at (1,0) {$ \top $};
\node (nb) at (1,-1) {$ \neg b $};
\draw [thick, dashed, ->] (T2) -- (nb);
\node (r3) at (1.3,-0.5) {$r_3$};
\draw [ultra thick] (1.6,0.3)--(1.6,-1.3);
\node (A) at (1.9,-0.5) {$B$};
\end{tikzpicture}
\caption{The arguments of Example \ref{eg:penguin_triangle}.}
\label{figure:penguin}
\end{center}
\end{figure}

Clearly $A$ and $B$ attack each other at their conclusions. Putting $<_D=\set{\pair{r_3,r_2}}$ into Equation \ref{eq:order_mapping_new_SP}, we get the following:
\begin{itemize}
\item For Algorithm \ref{alg:part1}, we have $S_0=\set{\infl}$, $N=3$ and $i=0,1,2,3$.

When $i=0$, $S_1=\es$, $\lambda=\infl$, $T^\lambda=\es$ and $R^\lambda=\set{r_1,r_3}\neq\es$. Therefore, $S_1=\set{\text{``}r_1\text{''},\text{``}r_3\text{''}}$. Notice $r_3$ is a $<_D$-maximal applicable rule, because even though $r_3<_D r_2$, $r_2$ is not applicable until $r_1$ is applied.

When $i=1$, $S_2=\es$, and either $\lambda=\text{``}r_1\text{''}$ or $\lambda=\text{``}r_3\text{''}$. The former gives $S_2=\set{\text{``}r_2 r_1\text{''}}$ and the latter gives $S_2=\set{\text{``}r_2 r_1\text{''},\text{``}r_1 r_3\text{''}}$.

When $i=2$, $S_3=\es$, $\lambda=\text{``}r_2r_1\text{''}$ or $\lambda=\text{``}r_1r_3\text{''}$. In the former case, $S_3=\set{\text{``}r_3r_2r_1\text{''}}$, and in the latter case, $S_3=\set{\text{``}r_3r_2r_1\text{''},\text{``}r_2r_1r_3\text{''}}$.

When $i=3$, we get $T^\lambda=\relsymb_d$ so $R^\lambda=\es$, halting the algorithm with output $S_3$.

\item Given $S_3$ as input to Algorithm \ref{alg:part2}, we get the intersection of the words ``$r_3r_2r_1$'' and ``$r_2r_1r_3$'' when converted to chains, giving $<_{SP}=\set{\pair{r_2,r_1}}$.
\end{itemize}
Therefore, Equation \ref{eq:order_mapping_new_SP}, returns $<_{SP}=F\pair{<_D}=\set{\pair{r_2,r_1}}$. If $<_D$ is arbitrary,\footnote{There are 19 partial orders on a set of three labelled elements.} we can repeat the above calculation and obtain the values of Equation \ref{eq:order_mapping_new_SP},\footnote{\label{fn:abbreviate_order} We abbreviate the total order $r_3<_D r_2<_D r_1$ as 321, $<_D:=\set{\pair{r_1,r_2},\pair{r_3,r_2}}$ as (12,32), and $<_D=\set{\pair{r_1,r_3}}$ as 13... etc. and the same applies to $<_{SP}$.} which are shown in Table \ref{table:values_of_SP_for_penguin}.

\begin{table}[h]
\begin{center}
\begin{tabular}{|c||c|}\hline
Values of input $<_D$ & Output $<_{SP}=F\pair{<_D}$\\\hline
(12,32), 32, 12, $\es$, 21 & 21 \\ \hline
123, 132, (13,23), (12,13), 13, 213 & 213 \\ \hline
31, (21,31) & (21,31) \\\hline
23, (21,23) & (21,23) \\\hline
312, (31,32), 321 & 321 \\\hline
231 & 231 \\\hline
\end{tabular}
\caption{The values of $<_{SP}$ given all possible $<_D$ of Example \ref{eg:penguin_triangle}.}
\label{table:values_of_SP_for_penguin}
\end{center}
\end{table}
\end{eg}

\vspace{-1cm}

\subsubsection{Properties of the Generalised SP Order}\label{sec:generalised_SP_properties}

We prove some properties of $F$ (Equation \ref{eq:order_mapping_new_SP} and Algorithms \ref{alg:part1} and \ref{alg:part2}) that will be useful in proving the representation theorem for the case where $<_D$ is a partial order (Theorem \ref{thm:rep_thm_nonlinear}). It can easily be shown that $F$ is a well-defined function from $PO\pair{\relsymb_d}$ to itself, where $PO\pair{\relsymb_d}$ is the set of all strict partial orders on $\relsymb_d$ (Section \ref{sec:notation}). We show that the function $F:PO\pair{\relsymb_d}\to PO\pair{\relsymb_d}$ indeed generalises the definitions in Section \ref{sec:pref}. Recall that $TO(X)$ is the set of all strict total orders on the set $X$.

\begin{lem}\label{lem:total_input_total_output}
If $<_D\in TO\pair{\relsymb_d}$ then we recover $<_{SP}$ defined in Section \ref{sec:pref}.
\end{lem}
\begin{proof}
If $<_D$ is total, then $R^\lambda$ on Algorithm \ref{alg:part1} Line \ref{alg_line:input_here} is singleton. This means Line \ref{alg_line:loop_over_choices} has only one choice in $R^\lambda$, so $S_i$ for all $0\leq i\leq N$ is singleton. Using the notation of Equation \ref{eq:SP_ord_def}, Algorithm \ref{alg:part1} returns $\set{\text{``}a_Na_{N-1}\ldots a_2a_1\text{''}}$. This gets transformed into Equation \ref{eq:SP_order} through Algorithm \ref{alg:part2}, which is $<_{SP}$ for the case of $<_D$ total.
\end{proof}

\begin{eg}\label{eg:penguin_triangle4}
(Example \ref{eg:penguin_triangle} continued) By restricting $F$ to $TO\pair{\relsymb_d}$, we obtain the following subtable of Table \ref{table:values_of_SP_for_penguin}, which indeed generalises the calculation for the case where the input $<_D$ is total.

\begin{table}[h]
\begin{center}
\begin{tabular}{|c||c|}\hline
Values of input $<_D$ & Output $<_{SP}=F\pair{<_D}$\\\hline
123, 132, 213 & 213 \\ \hline
312, 321 & 321 \\\hline
231 & 231 \\\hline
\end{tabular}
\caption{The values of $<_{SP}$ given the linear $<_D$ of Example \ref{eg:penguin_triangle4}.}
\label{table:values_of_SP_for_penguin_total}
\end{center}
\end{table}

\end{eg}


The next result shows that given the input $<_D$ in Algorithm \ref{alg:part1}, each string in the output set of Algorithm \ref{alg:part1}, when transformed into its corresponding total order on $\relsymb_d$, is the output of Equation \ref{eq:order_mapping_new_SP} for some linearisation $<^+_D$ of $<_D$.

\begin{thm}\label{thm:each_string_is_linearisation}
Consider Algorithm \ref{alg:part1} with input $<_D\in PO\pair{\relsymb_d}$. For each string $\sigma$ in the output set of Algorithm \ref{alg:part1}, let $<_\text{out}$ denote $\sigma$ transformed into a strict total order on $\relsymb_d$ (Algorithm, \ref{alg:part2}). For each such $<_\text{out}$ there exists a linearisation $<_D^+$ of $<_D$ such that $<_\text{out}=F\pair{<^+_D}$.
\end{thm}
\begin{proof}
If $<_D$ is itself total, then the output of Algorithm \ref{alg:part1} is singleton, which when converted to a chain by Algorithm \ref{alg:part2} gives $<_\text{out}=<_{SP}$. Therefore, there exists a linearisation of $<_D$, namely itself, such that $<_\text{out}=F\pair{<_D}$.

If $<_D$ is not total, then incomparable rules will cause $R^\lambda$ (Algorithm \ref{alg:part1} Line \ref{alg_line:input_here}) to not be singleton. Each element of $R^\lambda$ will form a distinct element of the output set of Algorithm \ref{alg:part1}. Choosing a given rule $r$ in $R^\lambda$ to append to the string can also be interpreted as a resolution of this incomparability of $<_D$ through a linearisation $<_D^+$ of $<_D$ that ranks $r$ higher than the alternative choices. Reasoning in this way in all cases whenever $R^\lambda$ is not singleton, we obtain a linearisation $<_D^+$ of $<_D$ such that $F\pair{<_D^+}$ corresponds to one of the elements in the output $S_N$ of Algorithm \ref{alg:part1}.
\end{proof}

\noindent It follows that Equation \ref{eq:order_mapping_new_SP} incorporates all possible linearisations of $<_D$ in the following manner.

\begin{cor}\label{cor:intersection_of_lin_of_input}
The output of Equation \ref{eq:order_mapping_new_SP} is equal to
\begin{align}
<_{SP}=F\pair{<_D}=\bigcap_{<^+_D\supseteq<_D\text{ total}}F\pair{<_D^+},\nonumber
\end{align}
where the intersection ranges over all linearisations of $<_D$.
\end{cor}
\begin{proof}
This is immediate from the definition of Algorithm \ref{alg:part2} and Theorem \ref{thm:each_string_is_linearisation}.
\end{proof}

 
\begin{eg}\label{eg:penguin_triangle5}
(Example \ref{eg:penguin_triangle4} continued) Consider $<_D=23$, which abbreviates $<_D=\set{\pair{r_2,r_3}}$. 23 has linearisations abbreviated as 123, 213 and 231 (see Footnote \ref{fn:abbreviate_order}, page \pageref{fn:abbreviate_order}). By Table \ref{table:values_of_SP_for_penguin_total}, these input linearisations returns, respectively, 213, 213 and 231. By Corollary \ref{cor:intersection_of_lin_of_input}, we get $<_{SP}$ to be the intersection of the sets representing the total orders 213 and 231. This gives $(21,23)$, which abbreviates $<_{SP}=\set{\pair{r_2,r_1},\pair{r_2,r_3}}$. This is consistent with Table \ref{table:values_of_SP_for_penguin_total}.
\end{eg}

\noindent We now relate the linearisations of the inputs and outputs of Table \ref{table:values_of_SP_for_penguin}.

\begin{thm}\label{thm:linearisation_square}
(The linearisation square) Let $<_D\:\in\: PO\pair{\relsymb_d}$ and $<_D^+$ be a linearisation of $<_D$. $F\pair{<_D^+}=:<^+_{SP}$ is a linearisation of $F\pair{<_D}=:<_{SP}$.
\end{thm}
\begin{proof}
Let $<_{SP}$ be as given. Let $<_D^+$ be a linearisation of $<_D$. Suppose $<_D^+$ is the input of Algorithm \ref{alg:part1}. This will give an output set consisting of a single string that Algorithm \ref{alg:part2} translates into some strict linear order $<_{SP}^+$ (say), by Theorem \ref{lem:total_input_total_output}. Upon input $<_D$ to Algorithm \ref{alg:part1}, the string corresponding to $<_{SP}^+$ will appear in the output set of Algorithm \ref{alg:part1}, because we can choose the rules in $R^\lambda$ (Line \ref{alg_line:input_here}) in accordance with the ranking of $<_D^+\:\supseteq\:<_D$. By Algorithm \ref{alg:part2}, $<_{SP}^+\:\supseteq F\pair{<_D}=<_{SP}$.
\end{proof}

\noindent The linearisation square can be expressed in the following commutative diagram:

\[
\begin{diagram}
\node{<_D} \arrow{e,b,J}{\text{linearisation}} \arrow{s,l,T}{\text{Equation \ref{eq:order_mapping_new_SP}}}
\node{<^+_D} \arrow{s,r,T}{\text{Equation \ref{eq:order_mapping_new_SP}}}\\
\node{F\pair{<_D}} \arrow{e,b,J}{\text{linearisation}}
\node{F\pair{<_D^+}}
\end{diagram}
\]

\noindent The linearisation square states that the function $F:PO\pair{\relsymb_d}\to PO\pair{\relsymb_d}$ preserves linearisations.

\begin{eg}\label{eg:penguin_triangle6}
(Example \ref{eg:penguin_triangle5} continued) Consider Table \ref{table:values_of_SP_for_penguin} again. Let $<_D=31$. Let $<_D^+=321$. We know that $<_{SP}=(21,31)$. We also know that $<_{SP}^+=321$. Clearly, 321 is a linearisation of $(21,31)$.
\end{eg}




\subsubsection{The Generalised Argument Preference Relation for \texorpdfstring{$<_D$}{<D} Partial}\label{sec:generalised_SP_order_on_args}

The following result states that changing the partial order PDL default priority $<$ to respect the logical dependencies of defaults while following the preference does not change the PDL extension. This generalises Lemma \ref{lem:LPDT_SP_order_keeps_extension_same} (page \pageref{lem:LPDT_SP_order_keeps_extension_same}) to the case where $<_D$ is not necessarily total.

\begin{lem}\label{lem:PDT_SP_order_keeps_extension_same}
Let $T:=\ang{D,W,<}$ and $T':=\ang{D,W,<'}$ be two PDTs such that $<\:\cong\:<_D$ and $F\pair{<_D}\:\cong\:<'$, then both PDTs have the same extensions
\end{lem}
\begin{proof}
Denote $Ext(T)$ and $Ext\pair{T'}$ to be the sets of extensions of the respective PDTs. We show that $Ext(T)=Ext\pair{T'}$.

$(\Rightarrow)$ Let $E\in Ext\pair{T}$ be arbitrary. This means $E$ is the unique extension of some LPDT $T^+:=\ang{D,W,<^+}$, where $<^+\:\supseteq\:<$ is a strict total order. Therefore, $<^+\:\cong\:<^+_D\:\supseteq\:<_D\:\cong\:<$, where $<^+_D$ is a linearisation of $<_D$. By the linearisation square (Theorem \ref{thm:linearisation_square}), $F\pair{<_D^+}=:<_{SP}^+$ is a linearisation of $F\pair{<_D}=:<_{SP}$. As $<'\:\cong\:<_{SP}$, then $<_{SP}^+\:\cong\:<^{'+}$, where $<^{'+}$ is some linearisation of $<'$. By Lemma \ref{lem:LPDT_SP_order_keeps_extension_same}, $E$ is also the unique extension of the LPDT $\ang{D,W,<^{'+}}$, which means $E$ is an extension of $T'=\ang{D,W,<'}$. Therefore, $E\in Ext\pair{T'}$.

$(\Leftarrow)$ Let $E\in Ext\pair{T'}$ be arbitrary. This means $E$ is the unique extension of some LPDT $T^+:=\ang{D,W,<^{'+}}$, where $<^{'+}\:\supseteq\:<'$ is a strict total order. As $<^{'+}\:\cong\:<_{SP}^+$, which is a linearistaion of $<_{SP}$, then by Theorem \ref{thm:each_string_is_linearisation}, there exists a linearisation $<_D^+$ of $<_D$ such that $<_D^+\:\mapsto\:<_{SP}^+$, given that $<_D\:\mapsto\:<_{SP}$. By Lemma \ref{lem:LPDT_SP_order_keeps_extension_same}, $E$ is the unique extension of $T^+:=\ang{D,W,<^+}$, where $<^+\:\cong\:<_D^+$, which means $E$ is an extension of $T$. Therefore, $E\in Ext\pair{T}$. Therefore, $Ext\pair{T}=Ext\pair{T'}$.
\end{proof}

\begin{eg}\label{eg:penguin_triangle8}
(Example \ref{eg:penguin_triangle6} continued) Recall the setup of Example \ref{eg:penguin_triangle}, where $W=\es$ and
\begin{align}
D=\set{d_1:=\frac{:a}{a},d_2:=\frac{a:b}{b},d_3:=\frac{:\neg b}{\neg b}}\nonumber.
\end{align}
Consider two strict partial orders on $D$, $<$ and $<'$, where $d_1<d_2$ only and $d_2<'d_1$ only. This gives us two PDTs $T=\ang{D,W,<}$ and $T'=\ang{D,W,<'}$. Let $<_D\:\cong\:<$. By Table \ref{table:values_of_SP_for_penguin}, $F\pair{<_D}=21$ so $F\pair{<_D}\:\cong\:<'$. Both PDTs $T$ and $T'$ have the same extensions. In the case of $T$, we have linearisations $312$, $132$ and $123$, with the first linearisation giving $E_1:= Th\pair{\set{a,b}}$ and the latter two linearisations giving $E_2:=Th\pair{\set{a,\neg b}}$. In the case of $T'$, we have linearisations 321, 231 and 213, with the first linearisation giving $E_1$ and the latter two linearisations giving $E_2$. Therefore, both $T$ and $T'$ have the same extensions.
\end{eg}

We can now define the associated set comparison relation from this new $<_{SP}$ just like Equation \ref{eq:SP_set_comparison}: for $\Gamma,\:\Gamma'\subseteq_\text{fin}\relsymb_d$,
\begin{align}\label{eq:new_SP_set_comparison}
\Gamma\ordneq_{SP}\Gamma'\Leftrightarrow\pair{\exists x\in\Gamma-\Gamma'}\pair{\forall y\in\Gamma'-\Gamma}x<_{SP}y,
\end{align}
where given the partial order default priority $<$, order isomorphic to $<_D$ (Equation \ref{eq:bij_defaults_def_rules}), $<_{SP}=F\pair{<_D}$ is the output of Equation \ref{eq:order_mapping_new_SP}. The associated strict argument preference is, for $A,B\in\alg$,
\begin{align}\label{eq:generalised_argument_preference_SP}
A\prec_{SP}B\Leftrightarrow DR(A)\ordneq_{SP} DR(B).
\end{align}
The associated non-strict argument preference is
\begin{align}\label{eq:generalised_argument_preference_SP_non_strict}
A\precsim_{SP}B\Leftrightarrow\sqbra{DR(A)\ordneq_{SP}DR(B)\text{ or }DR(A)=DR(B)},
\end{align}
These equations are the same as Equations \ref{eq:SP_arg_pref} and \ref{eq:SP_arg_pref_non_strict} respectively.\footnote{It can be shown that $\precsim_{SP}$ in the partial order case is \textit{not} transitive, unlike in the total case (Lemma \ref{lem:SP_total_preorder_when_total}, also recall Footnote \ref{fn:acyclic}), but is acyclic. We will discuss this in future work.}

\begin{eg}\label{eg:penguin_stables}
(Example \ref{eg:penguin_triangle} continued) Suppose we define $\prec_{SP}$ by Equation \ref{eq:generalised_argument_preference_SP} with this new $<_{SP}$. We have both $A\not\prec_{SP}B$ and $B\not\prec_{SP}A$. This means there are two stable extensions: $\ext_1$ which contains $A_0$ and $A$, and $\ext_2$ which contains $A_0$ and $B$. The conclusion set of these stable extensions correspond respectively to $E_1$ and $E_2$. 

\end{eg}


So given a PDT $T$ where the default priority $<$ is not necessarily total, we construct the set of arguments $\alg$ and define the attack relation $\attk$ as in Section \ref{sec:lift_total_assumption}. We define the non-strict argument preference relation $\precsim_{SP}$ as in Equations \ref{eq:generalised_argument_preference_SP_non_strict}, \ref{eq:generalised_argument_preference_SP}, \ref{eq:new_SP_set_comparison} and \ref{eq:order_mapping_new_SP}, given $<\:\cong\:<_D$. The \textit{attack graph of the PDT $T$} is the structure $AG(T):=\ang{\alg,\attk,\precsim_{SP}}$. The \textit{defeat graph of the PDT $(T)$} is the structure $DG(T):=\ang{\alg,\defeat}$, where $\defeat$ is defined as in Equation \ref{eq:ASPIC+_general_defeat} under the argument preference relation $\precsim_{SP}$.

\subsection{The Representation Theorem for Partial Order Default Priorities}\label{sec:rep_thm_nonlinear}

We now generalise Theorem \ref{thm:rep_thm} to the case where $<_D$ is a partial order. Our proof strategy is to leverage as much of Theorem \ref{thm:rep_thm} as possible. The difference here is that our default priority $<\:\cong\:<_D$ is now partial. In the previous section, we saw how the linearisation square (Theorem \ref{thm:linearisation_square}) related the lift $<_D\:\mapsto\:<_{SP}$ to the lift $<_D^+\:\mapsto\:<_{SP}^+$ where $<_D^+$ is a linearisation of $<_D$. We now apply the linearisation square to relate partial order $<_D$ with their linearisations $<_D^+$ in the case of the defeat graphs generated and their stable extensions. Specifically, if $<_D^+$ is a linearisation of $<_D$, then the defeat graph of the former is a spanning subgraph of the latter. Further, the unique stable extension in the former case is still a well-defined stable extension in the latter case. The next two sections establish these results, which will then be used to prove the generalised representation theorem.

\subsubsection{Linearisation of the Argument Preference Relation and Spanning Subgraphs}

Recall from graph theory that $G':=\ang{V,E'}$ is a \textit{spanning subgraph} of $G:=\ang{V,E}$ iff $E'\subseteq E$, and we write $G'\subseteq_\text{span}G$. For spanning argument sub-frameworks, stable extensions are preserved as long as you do not add conflicts between arguments in the stable extension.

\begin{lem}\label{lem:stable_extensions_of_spanning_subgraphs}
Let $AF:=\ang{A,\to}$ be an abstract argumentation framework. Let $AF':=\ang{A,\to'}$ be a spanning subgraph of $AF$. If $\ext$ is a stable extension of $AF'$ and $\ext^2\cap\to=\es$, then $\ext$ is also a stable extension of $AF$.
\end{lem}
\begin{proof}
By assumption, $\ext$ is cf because it is a stable extension. Let $b\notin\ext$, then $\ext\to' b$, but as $\to'\:\subseteq\:\to$ by definition, we also have $\ext\to b$. Therefore, $\ext$ is a stable extension of $AF$.
\end{proof}

\noindent Linearising the structure preference order $<_{SP}$ on the rules also linearises the set comparison relation $\ordneq_{SP}$ and the argument preference $\precsim_{SP}$ by Lemma \ref{lem:disj_eli_is_transitive_over_chain}.

\begin{lem}\label{lem:linearisation_of_SP_order_linearises_arg_pref}
Let $<_{SP}$ be the output of Equation \ref{eq:order_mapping_new_SP} for some input $<_D$. Let $<_{SP}^+$ be a linearisation of $<_{SP}$. Let the binary relations on $\powfin\pair{\relsymb_d}$, $\ordneq_{SP}$ and $\ordneq_{SP}^+$, be obtained by applying Equation \ref{eq:new_SP_set_comparison} to $<_{SP}$ and $<_{SP}^+$ respectively. Then
\begin{enumerate}
\item $\ordneq_{SP}\:\subseteq\:\ordneq_{SP}^+$,
\item $\prec_{SP}\:\subseteq\:\prec_{SP}^+$, where $\prec_{SP}$ is the strict part of Equation \ref{eq:generalised_argument_preference_SP} on $\ordneq_{SP}$ and analogously for $\prec_{SP}^+$ on $\ordneq_{SP}^+$, and
\item $\precsim_{SP}\:\subseteq\:\precsim_{SP}^+$.
\end{enumerate}
\end{lem}
\begin{proof}
(1) Let $\Gamma,\:\Gamma'\in\powfin\pair{\relsymb_d}$ be arbitrary. Assume $\Gamma\ordneq_{SP}\Gamma'$. Then by Equation \ref{eq:new_SP_set_comparison}, this is equivalent to $\pair{\exists x\in\Gamma-\Gamma'}\pair{\forall y\in\Gamma'-\Gamma}x<_{SP} y$, which by our assumption implies that $\pair{\exists x\in\Gamma-\Gamma'}\pair{\forall y\in\Gamma'-\Gamma}x<^+_{SP} y$, and hence $\Gamma\ordneq^+_{SP}\Gamma'$. (2) Let $A,B\in\alg$ be arbitrary. We have $A\prec_{SP}B\Leftrightarrow DR(A)\ordneq_{SP}DR(B)$. From the first result, $DR(A)\ordneq_{SP}DR(B)$ then $DR(A)\ordneq_{SP}^+DR(B)$. (3) This follows trivially from Equation \ref{eq:generalised_argument_preference_SP}.
\end{proof}

We now prove the converse of Lemma \ref{lem:linearisation_of_SP_order_linearises_arg_pref}.

\begin{lem}\label{lem:spanning_defeat_graphs}
Let $T$ be a PDT, $<\:\cong\:<_D$ and $<_{SP}=F\pair{<_D}$. Let $<_{SP}^+$ be a linearisation of $<_{SP}$. Let $\precsim_{SP}$ and $\precsim_{SP}^+$ be the lift of $<_{SP}$ and $<_{SP}^+$ respectively from $\relsymb_d$ to an argument preference relation on $\alg$ in the usual way (Equations \ref{eq:new_SP_set_comparison} and \ref{eq:generalised_argument_preference_SP}). Let $DG(T):=\ang{\alg,\defeat}$ and $DG^+\pair{T}:=\ang{\alg,\defeat^+}$ be the respective defeat graphs of the attack graphs $AG(T):=\ang{\alg,\attk,\precsim_{SP}}$ and $AG^+\pair{T}:=\ang{\alg,\attk,\precsim_{SP}^+}$. Then $DG^+\pair{T}\subseteq_\text{span} DG(T)$.
\end{lem}
\begin{proof}
Clearly both $DG(T)$ and $DG^+(T)$ have the same vertex set $\alg$. We show that $\defeat^+\:\subseteq\:\defeat$. Let $A,B\in\alg$ be arbitrary such that $A\defeat^+ B$. Suppose $B'\subarg B$ is the argument defeated by $A$ at its top rule. By Equation \ref{eq:ASPIC+_general_defeat}, $A\attk B'$ and $A\not\prec_{SP}^+ B'$. It is sufficient to show that $A\not\prec_{SP} B'$. By Lemma \ref{lem:linearisation_of_SP_order_linearises_arg_pref}, we have $\prec_{SP}\:\subseteq\:\prec_{SP}^+$ meaning that if $A\not\prec_{SP}^+B'$ then $A\not\prec_{SP}B'$. Hence, $A\defeat B'$ and so $A\defeat B$. It follows that $\defeat^+\:\subseteq\:\defeat$.
\end{proof}

\subsubsection{Existence of Stable Extensions in the Partial Order Case}

For a PDT $T$, its defeat graph $DG\pair{T}$ has stable extensions that do not have to be unique.

\begin{thm}\label{thm:stb_ext_exists_nonlinear_case}
Let $T$ be a PDT, with attack graph $AG(T)$ and defeat graph $DG(T)$ where, as usual, $<\:\cong\:<_D$, $<_{SP}=F\pair{<_D}$ by Equation \ref{eq:order_mapping_new_SP}, and $\precsim_{SP}$ is defined from $<_{SP}$ using Equations \ref{eq:new_SP_set_comparison} and \ref{eq:generalised_argument_preference_SP}. The defeat graph $DG(T)$ has a stable extension that is not in general unique.
\end{thm}
\begin{proof}
The PDT $T$ has some extension $E$, which is the unique stable extension of an LPDT $T^+:=\ang{D,W,<^+}$, where $<^+$ is the linearisation of $<$ that generates $E$. Given that $<\:\cong\:<_D$ and $<^+\:\cong\:<_D^+$, we know that $<_D^+$ is also a linearisation of $<_D$. By the linearisation square (Theorem \ref{thm:linearisation_square}), $F\pair{<_D^+}=:<_{SP}^+$ is a linearisation of $F\pair{<_D}=:<_{SP}$.

The LPDT $T^+$ has an attack graph $AG\pair{T^+}:=\ang{\alg,\attk,\precsim_{SP}^+}$, where $\precsim_{SP}^+$ is calculated from $<_{SP}^+$. As $<_{SP}^+$ is a linearisation of $<_{SP}$, the defeat graph of $T^+$, $DG\pair{T^+}:=\ang{\alg,\defeat^+}$ is a spanning subgraph of $DG(T)$ by Lemma \ref{lem:spanning_defeat_graphs}.

The LPDT $T^+$ has $E$ as its unique PDE. By the representation theorem for LPDTs (Theorem \ref{thm:rep_thm}), there exists a unique stable extension $\ext$ of $DG\pair{T^+}$ such that $Conc\pair{\ext}=E$. By Corollary \ref{cor:stb_ext_acf}, $\ext^2\:\cap\:\attk=\es$. $DG(T)$ differs from $DG\pair{T^+}$ by their argument preference relations, as the attack relation is the same for both. As no attacks are introduced to $\ext$, $\ext$ is also cf in $DG(T)$ by Equation \ref{eq:ASPIC+_general_defeat}. Therefore, by Lemma \ref{lem:stable_extensions_of_spanning_subgraphs}, $\ext$ is also a stable extension of $DG(T)$. Therefore, $DG(T)$ also has $\ext$ as a stable extension.

To show that this stable extension is not in general unique, consider the PDT $\ang{\set{d_1:=\frac{:a}{a},d_2:=\frac{:\neg a}{\neg a}},\es,\es}$. We can construct the arguments $A:=[\Rightarrow a]$ and $B:=[\Rightarrow\neg a]$, which symmetrically attack each other on their conclusions. As $<=\es$, we have $<_D=<_{SP}=\es$ from Equation \ref{eq:order_mapping_new_SP}. Therefore, $A\not\prec_{SP} B$ and $B\not\prec_{SP} A$ and we have two stable extensions: one where $A$ is justified (and $B$ is not justified), and the other where $B$ is justified (and $A$ is not justified).
\end{proof}

\subsubsection{Proof of the Representation Theorem}

We prove the representation theorem in this section. Our technique is to relate a partial order on the defaults $<$ with one of its possible linearisations $<^+$, and invoke the first representation theorem (Theorem \ref{thm:rep_thm}) for this linearisation. We know given an extension $E$ of some PDT $T$ there is a linearisation $<^+$ of $<$ generating $E$. We now establish an analogous result on the side of argumentation: for every stable extension $\ext$ of $DG\pair{T}$, there is a linearisation $<_{SP}^+$ of $<_{SP}$ on $\relsymb_d$ such that $<_{SP}^+$ constructs $\ext$ via Algorithm \ref{alg:gen_stab_ext}.

We show $<_{SP}^+$ exists given $<_{SP}$ by construction, which will make use of a \textit{partial linearisation} of an order $<$ on $P$. Let $\ang{P,<}$ be a poset and $U\subseteq P$. Let $<_U:=<\cap U^2$ be the partial order $<$ restricted to $U$. Let $<_U^+$ be a linearisation of $<_U$ on $U$. We define
\begin{align}\label{eq:par_lin_U}
<_U^\text{par}:=TrCl\pair{<_U^+\cup <},
\end{align}
where $TrCl$ denotes the transitive closure. It can be shown that given $\ang{P,<}$ and $U$, $<_U^\text{par}$ is a strict partial order on $P$ extending $<$, which is linear when restricted to the set $U$. Further, $<_U^\text{par}$ is not unique because there could be many possible linearisations of $<$ over $U$.

\begin{lem}\label{lem:each_stb_ext_has_lin_gen_it}
Let $T$ be a PDT with defeat graph $DG(T)$ where $<_{SP}$ lifts to $\precsim_{SP}$. Let $\ext$ be a stable extension of $DG(T)$. There exists a linearisation $<_{SP}^+$ of $<_{SP}$ such that $\ext$ is the output of Algorithm \ref{alg:gen_stab_ext} with $<_{SP}^+$ as input.
\end{lem}
\begin{proof}
Given $\ext$, let $R^+:=DR\pair{\ext}$ and $R^-=\relsymb_d-R^+$. Define $<_0:=<_{SP,R^-}^\text{par}$, which is Equation \ref{eq:par_lin_U} with $\ang{P,<}=\ang{\relsymb_d,<_{SP}}$ and $U=R^-$. As $\relsymb_d$ is finite, WLOG let $R^-:=\set{s_1,s_2,\ldots,s_m}$ for some $m\in\nat$, such that $i<j\Leftrightarrow s_j<_0 s_i$. so $s_1$ is $<_0$-greatest on $R^-$. For $1\leq i\leq m$, define the set
\begin{align}\label{eq:Ui}
nonlower_{<_{i-1}}(s_i):=:U_i:=\set{r\in\relsymb_d\:\vline\:r\not <_{i-1} s_i}.
\end{align}
For $0\leq i\leq m-1$ we extend $<_i$ to a new partial order $<_{i+1}:=<_{i,U_{i+1}}^\text{par}$, which is Equation \ref{eq:par_lin_U} with $\ang{P,<}=\ang{\relsymb_d,<_i}$ and $U=U_i$ (Equation \ref{eq:Ui}), \textit{such that $s_i$ is the $<_i$-least element on the set $U_i$}. Once we reach $<_m$ we take a final linearisation of $<_m$ to get $<_{SP}^+$.\footnote{Simple examples can be devised where $<_m$ is not a total order on $\relsymb_d$.} This construction therefore gives an increasing sequence of partial orders $<_{SP}\subseteq <_0\subseteq <_1\subseteq\cdots\subseteq <_m\subseteq <_{SP}^+.$ on $\relsymb_d$. Clearly, $<_{SP}^+$ is a well-defined linearisation of $<_{SP}$ by construction. 

To show that $<_{SP}^+$ generates $\ext$ when input into Algorithm \ref{alg:gen_stab_ext}, consider $s_i\in R^-$. We assume no defeasible rule is unnecessary, i.e. $\pair{\forall r\in\relsymb_d}\pair{\exists A\in\alg}r\in DR(A)$. In other words, each defeasible rule is used in some argument.\footnote{This is a fair assumption to make given that PDTs typically do not have defaults that are excluded from all extensions.} Therefore, there is some argument $B_i\in\alg$ such that $TopRule\pair{B_i}=s_i$. By how $<_{SP}$ is defined, $s_i$ is the $<_{SP}$-least rule in $DR\pair{B_i}$. By Lemma \ref{lem:rules_in_stb_ext}, $B_i\notin\ext$ so $\ext\defeat B_i$. Let $A_i\defeat B_i$ for $A_i\in\ext$. This would mean $A_i\not\prec_{SP} B_i$.

We show that for the set of defeasible rules $U_i$ associated with rule $s_i$ as defined in Equation \ref{eq:Ui}, $\pair{\forall r\in DR\pair{A_i}}r\in U_i$. Assume for contradiction that $\pair{\exists r\in DR\pair{A_i}}r\notin U_i$, then there is some $r_0\in DR\pair{A_i}$, $r_0<_{i-1}s_i$. By the properties of Equation \ref{eq:par_lin_U}, we can show that $r_0<_{i-1}s_i<_{i-1}s_{i-1}<_{i-1}\cdots<_{i-1}s_1$. Therefore, $r_0$ cannot be in the sets $U_j$ for any $j<i$, and so could not have been linearised above $s_j$ for $j<i$ in any of the previous stages. Therefore, $r_0<_{SP} s_i$. As $s_i$ is $<_{SP}$-least in $DR\pair{B_i}$ by being the top rule of $B_i$, $r_0\notin DR\pair{B_i}$ and hence there is some rule, $r_0\in DR\pair{A_i}-DR\pair{B_i}$, such that for all rules $x\in DR\pair{B_i}-DR\pair{A_i}$, $r_0<_{SP}x$. Therfore, $A_i\prec_{SP}B_i$ -- contradiction, as $A_i\defeat B_i$. Therefore, all the defeasible rules of $A_i$ are in $U_i$, and in the linearisation process where $<_{i-1}$ is linearised over $U_i$ into $<_i$ such that $s_i$ is $<_i$-minimal in $U_i$, we have ensured that at least one defeater of $B_i$ will be constructed by Algorithm \ref{alg:gen_stab_ext} and included in $\ext$ prior to the consideration of the rule $s_i$. As $i$ is arbitrary, we have shown that the final linearisation $<_{SP}^+$ ensures that all arguments containing rules in $R^-$ are defeated and excluded from $\ext$. Therefore, Algorithm \ref{alg:gen_stab_ext}, upon input from $<_{SP}^+$, generates exactly $\ext$.
\end{proof}

\noindent We give two concrete examples of the construction of $<_{SP}$ in Lemma \ref{lem:each_stb_ext_has_lin_gen_it}.

\begin{eg}
Consider the PDT $\ang{\set{d_1:=\frac{:a}{a},d_2:=\frac{:\neg a}{\neg a}},\es,\es}$ from the proof of Theorem \ref{thm:stb_ext_exists_nonlinear_case}. Translating to argumentation, there are two arguments $A:=[\Rightarrow a]$ and $B:=[\Rightarrow\neg a]$ which attack each other at their conclusions. Clearly, $<_{SP}=\es$ and there are two stable extensions: $\ext_1$ such that $A\in\ext_1$ and $\ext_2$ such that $B\in\ext_2$. Suppose we choose the stable extension $\ext_1$ and construct a linearisation of $\es$ that generates $\ext_1$. We have $R^+=\set{r_1}$ and $R^-=\set{r_2}$. Vacuously, $<_{SP}$ is already linear on $\set{r_2}$ so $<_0=<_{SP}$. We then consider $nonlower_{<_{SP}}(r_2)=\set{r_1,r_2}$. We linearise $<_{SP}$ such that $r_2$ is smaller than all other elements in $nonlower_{<_{SP}}(r_2)$, so $r_2<_{SP}^+ r_1$. This is indeed the linearisation of $<_{SP}$ that generates $\ext_1$.
\end{eg}

\begin{eg}\label{eg:four_args}
Let $\mathcal{K}_n=\es$ and $r_1:=(\Rightarrow\neg b)$, $r_2:=\pair{\Rightarrow a}$, $r_3:=(a \Rightarrow b)$, $r_4:=(\Rightarrow c)$, $r_5:=(c \Rightarrow\neg b)$, $r_6:=(\Rightarrow b)$. Define $A:=[\Rightarrow\neg b]$, $B:=[[\Rightarrow a]\Rightarrow b]$, $C:=[[\Rightarrow c]\Rightarrow\neg b]$ and $D:=[\Rightarrow b]$. We illustrate these arguments in Figure \ref{figure:four_args}.

\begin{figure}[h]
\begin{center}
\begin{tikzpicture}[scale = 1]
\node (T1) at (0,0) {$ \top $};
\node (nb) at (0,-1) {$ \neg b $};
\draw [thick, dashed, ->] (T1) -- (nb);
\node (r1) at (0.3,-0.5) {$r_1$};
\draw [ultra thick] (-0.3,0.3)--(-0.3,-1.3);
\node (A) at (-0.6,-0.5) {$A$};

\node (T2) at (2,0) {$ \top $};
\node (a) at (2,-1) {$ a $};
\draw [thick, dashed, ->] (T2) -- (a);
\node (r2) at (2.3,-0.5) {$r_2$};
\node (b) at (2,-2) {$ b $};
\draw [thick, dashed, ->] (a) -- (b);
\node (r3) at (2.3,-1.5) {$r_3$};
\draw [ultra thick] (1.7,0.3)--(1.7,-2.3);
\node (B) at (1.4,-1) {$B$};

\node (T3) at (4,0) {$ \top $};
\node (c) at (4,-1) {$ c $};
\draw [thick, dashed, ->] (T3) -- (c);
\node (r4) at (4.3,-0.5) {$r_4$};
\node (nb2) at (4,-2) {$ \neg b $};
\draw [thick, dashed, ->] (c) -- (nb2);
\node (r5) at (4.3,-1.5) {$r_5$};
\draw [ultra thick] (3.7,0.3)--(3.7,-2.3);
\node (C) at (3.4,-1) {$C$};

\node (T4) at (6,0) {$ \top $};
\node (b) at (6,-1) {$ b $};
\draw [thick, dashed, ->] (T4) -- (b);
\node (r6) at (6.3,-0.5) {$r_6$};
\draw [ultra thick] (5.7,0.3)--(5.7,-1.3);
\node (D) at (5.4,-0.5) {$D$};
\end{tikzpicture}
\caption{The arguments of Example \ref{eg:four_args}.}
\label{figure:four_args}
\end{center}
\end{figure}
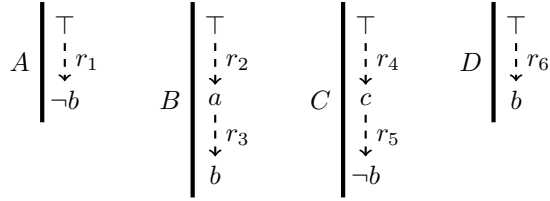

\noindent Suppose $<_{SP}$ is such that $r_6<_{SP} r_5<_{SP} r_4$, $r_5<_{SP}r_3<_{SP}r_1$ and $r_3<_{SP}r_2$. It can be shown that $D\prec_{SP} C\prec_{SP}B\prec_{SP} A$ hence $A\defeat B\defeat C\defeat D$ (notice $A\defeat D$ as well). The stable extension therefore contains $A$, $C$ and $[\Rightarrow a]$, so $R^-=\set{r_3,r_6}$. As $r_6<_{SP} r_3$, $<_{SP}$ is already linear on $R^-$, so $<_0=<_{SP}$. Now consider $nonlower_{<_0}(r_3)=\set{r_1,r_2,r_3,r_4}$ and linearise $<_{SP}$ over this set such that $r_3$ is the smallest element in $nonlower_{<_0}(r_3)$, so suppose $<_1$ is $r_3<_1 r_1<_1 r_2 <_1 r_4$. Now consider $r_6$, but $nonlower_{<_1}(r_6)=\relsymb_d$ and is already linear, so we take $<_{SP}^+$ to be the chain 653124 when written in abbreviated form (see Footnote \ref{fn:abbreviate_order}, page \pageref{fn:abbreviate_order}). This $<_{SP}^+$, when input into Algorithm \ref{alg:gen_stab_ext}, will generate $\ext$.
\end{eg}


We now apply Lemma \ref{lem:each_stb_ext_has_lin_gen_it} to prove a more general representation theorem.

\begin{thm}\label{thm:rep_thm_nonlinear}
(The Representation Theorem for Partial Order Default Priorities) Let $AG(T)$ be the attack graph corresponding to a PDT $T$, where the default priority $<$ is not necessarily total, with defeat graph $DG(T)$ under $\precsim_{SP}$ as defined by Equations \ref{eq:order_mapping_new_SP}, \ref{eq:new_SP_set_comparison} and \ref{eq:generalised_argument_preference_SP}.
\begin{enumerate}
\item Let $E$ be an extension of $T$. Then there exists a corresponding stable extension $\ext\subseteq\alg$ of $DG(T)$ such that $Conc\pair{\ext}=E$.
\item Let $\ext\subseteq\alg$ be a stable extension of $DG(T)$, then $Conc(\ext)$ is an extension of $T$.
\end{enumerate}
\end{thm}
\begin{proof}
\textbf{Proof of part 1:} Let $E$ be an extension of $T$, then there exists a LPDT $T^+:=\ang{D,W,<^+}$ where $<^+$ is a linerisation of $<$ that generates the extension $E$ (Equation \ref{eq:ext_ind}). Consider the defeat graph $DG\pair{T^+}$. By Theorem \ref{thm:rep_thm}, there exists a stable extension $\ext$ of $DG\pair{T^+}$ such that $Conc\pair{\ext}=E$. Arguing as in the proof of Theorem \ref{thm:stb_ext_exists_nonlinear_case} where $DG\pair{T^+}\subseteq_\text{span}DG(T)$, $\ext$ is also a stable extension of $DG(T)$, and it satisfies $Conc\pair{\ext}=E$.

\textbf{Proof of part 2:} Let $\ext$ be a stable extension, which exists by Theorem \ref{thm:stb_ext_exists_nonlinear_case}. By Lemma \ref{lem:each_stb_ext_has_lin_gen_it}, there is a linearisation $<_{SP}^+$ of $<_{SP}$ such that Algorithm \ref{alg:gen_stab_ext} returns $\ext$ upon input $<_{SP}^+$. Consider the LPDT $T^+_1:=\ang{D,W,<_1^+}$, where $<_1^+\:\cong\:<_{SP}^{+}$. By Section \ref{sec:ASPIC+_to_PDL}, this has a defeat graph $DG\pair{T^+_1}$ with unique stable extension $\ext$. The set $Conc\pair{\ext}$ is an extension of $T^+_1$ by Theorem \ref{thm:rep_thm}. Clearly $Conc\pair{\ext}$ is also an extension of $\ang{D,W,<_1}$, where $<_1\:\cong\:<_{SP}$. By Lemma \ref{lem:PDT_SP_order_keeps_extension_same}, $Conc\pair{\ext}$ is also an extension of $\ang{D,W,<}=T$.
\end{proof}

\noindent Under the generalised SP argument preference $\precsim_{SP}$, this representation theorem means that PDL is also sound and complete with respect to its argumentation semantics in the case where the default priority $<$ is not necessarily total.

\subsection{Satisfaction of Rationality Postulates}\label{sec:rationality_nonlinear}

In this section we will state and prove a version of Theorem \ref{thm:PDL_is_rational_direct}, which is that the rationality postulates \cite{Caminada:07} hold for the stable extensions of the defeat graph, instead for all complete extensions. We will discuss the possibility for a general proof in Section \ref{sec:discussion_conclusions}.

\begin{thm}\label{thm:weakened_rationality_theorem}
(Rationality Theorem for Stable Extensions) Let $T:=\ang{D,W,<}$ be a PDT. Let its corresponding attack graph be $AG(T):=\ang{\alg,\attk,\precsim_{SP}}$ where $<\:\cong\:<_D\:\mapsto\:<_{SP}$ by Equation \ref{eq:order_mapping_new_SP}, and $\precsim_{SP}$ is defined in terms of $<_{SP}$ using Equations \ref{eq:new_SP_set_comparison} and \ref{eq:generalised_argument_preference_SP}. Let $DG(T):=\ang{\alg,\defeat}$ be the corresponding defeat graph. All stable extensions of $DG(T)$ satisfy the Caminada-Amgoud rationality postulates.
\end{thm}
\begin{proof}
Given $T$, let $\ext$ be any stable extension of $DG(T)$.
\begin{enumerate}
\item To show that $\ext$ is subargument closed, let $A\in\ext$ and let $B\subarg A$. Assume for contradiction that $B\notin\ext$, then $\ext\defeat B$ and hence there is some $C\in\ext$ such that $C\defeat B$. Therefore, $C\defeat A$. This means $\ext$ is not cf -- contradiction. Therefore, $B\in\ext$ as well, and $\ext$ is thus subargument closed.
\item Theorem \ref{thm:rep_thm_nonlinear} states that $Conc\pair{\ext}$ is an extension of $T$, which is deductively closed. Therefore, $Conc\pair{\ext}$ is closed under strict rules.
\item As $W$ is consistent and $Conc\pair{\ext}$ is an extension of $T$, $Conc\pair{\ext}$ must also be consistent and its deductive closure is consistent.
\end{enumerate}
This shows the result.
\end{proof}

\noindent In conclusion, all stable extensions are normatively rational. This generalises the rationality theorem (Theorem \ref{thm:PDL_is_rational_direct}) to the case where $<_D$ is partial, although \textit{only for stable extensions}.

\subsection{Summary}

In this section, we have generalised our sound and complete instantiation of ASPIC$^+$ to PDL to the case where the default priority is not necessarily a total order. The main challenge is generalising the SP argument preference $\precsim_{SP}$ from a total default priority to a partial default priority. We devise a sorting $F\pair{<_D}=<_{SP}$ such that $<_{SP}$ sorts $<_D$ in a way that respects the argument structure, the defeasible rule preference $<_D$, and the incomparability of rules (Section \ref{sec:partial_order_pref_part2}). This preference has the correct properties to preserve the correspondence between the inferences of the underlying PDT and the conclusions of justified arguments (Section \ref{sec:rep_thm_nonlinear}, Theorem \ref{thm:rep_thm_nonlinear}). We have also shown that each stable extension satisfies the rationality postulates (Section \ref{sec:rationality_nonlinear}, Theorem \ref{thm:weakened_rationality_theorem}).

\section{Conclusions}\label{sec:discussion_conclusions}

We have endowed Brewka's PDL \cite{Brewka:94} with argumentation semantics using ASPIC$^+$ \cite{sanjay:13}. This is achieved by representing PDL in ASPIC$^+$ (Sections \ref{sec:instantiation_choice_of_variables} and \ref{sec:lift_total_assumption}), discussing which argument preference relations can be suitable for the correspondence of inferences (Sections \ref{sec:pref} and \ref{sec:partial_order_pref_part2}), proving that the inferences do correspond under an appropriate preference relation $\precsim_{SP}$ (Sections \ref{sec:rep_thm} and \ref{sec:rep_thm_nonlinear}), and that this instantiation is rational (Sections \ref{sec:normative_rationality_of_PDL_inst} and \ref{sec:rationality_nonlinear}). As explained in Section \ref{sec:intro}, this allows us to interpret the inferences of PDL as conclusions of justified arguments, clarifying the reasons for accepting or rejecting a conclusion. The argumentative characterisation of PDL provides for distributed reasoning in the course of deliberation and persuasion dialogues. This would allow BOID agents with PDL representations of mental attitudes to exchange arguments and counterarguments when deliberating about which goals to select, and thus which actions to pursue \cite{BOID:02}.

\subsection{Related Work}

As mentioned in Section \ref{sec:intro}, there are many existing argumentative characterisations of non-monotonic logics (e.g. \cite{Dung:95,Governatori:04}). However, there has been relatively little work in using defeasible rules to represent defaults, because  the defeasible components of arguments are often captured in the premises \cite{Bondarenko:97,sanjay:13}. Reiter's default logic (DL) \cite{Reiter:80}, as a partial special case\footnote{i.e. in the case where all defaults are normal defaults.} of Brewka's PDL \cite[Proposition 6]{Brewka:94}, has been endowed with sound and complete argumentation semantics by Dung \cite[Section 4.1]{Dung:95}. However, DL cannot handle priorities and as a result draws counter-intuitive inferences. We know that conflicts between defaults often occur and priorities are an intuitive and high-level way of resolving such conflicts \cite{Brewka:94,Brewka:00}. It is therefore important to investigate how preferences can also be incorporated into any argumentation semantics. ASPIC$^+$ is a good framework to achieve this because it is designed to handle preferences.

\subsection{Future Work}

Brewka's preferred subtheories (PS) \cite{Brewka:89} has been endowed with argumentation semantics by Modgil and Prakken using ASPIC$^+$ \cite[Section 5.3.2]{sanjay:13}. Given that PS is a special case of PDL \cite{Brewka:89}, it is interesting to see how the argumentation semantics are related. It can be shown that instantiating the argumentation semantics of PDL in this paper to the case of supernormal defaults and empty facts will recover an argumentation semantics isomorphic to the argumentation semantics of Modgil and Prakken. However, whereas Modgil and Prakken assume that arguments must be consistent, the results of Section \ref{sec:inconsistent_args} lifts this assumption when we specialise our argumentation semantics. We will articulate this in future work.

It will also be interesting to see how an argumentation semantics for Reiter's normal DL \cite[Section 3]{Reiter:80} can be recovered by setting $<_D=\es$ \cite[Proposition 6]{Brewka:94}, and comparing this to Dung's argumentation semantics for DL. However, Dung's argumentation semantics also accommodates non-normal defaults. How would ASPIC$^+$ incorporate non-normal defaults? At first glance it should involve the naming function and undercuts (Section \ref{sec:rev_ASPIC+}), but how can soundness and completeness be proven? How can the argumentation semantics help us understand the interaction of explicit default priority relations with the implicit priority of non-normal defaults \cite{Reiter:81}? Future work will explore further properties of this argumentation semantics.

ASPIC$^+$ can be used to generalise PDL. For example, we know that extensions do not have to exist for non-normal default logic, which corresponds to the failure for stable extensions to exist in the argumentation semantics \cite[Section 4.1]{Dung:95}. We can then consider the justified arguments under different Dung semantics, but what would these other notions of justified arguments mean for PDL?

Another reason for considering different Dung semantics is to show whether the rationality postualtes holds for complete extensions in general. So far we have shown a special case of rationality for the stable extensions only (Section \ref{sec:rationality_nonlinear}). What would the complete extensions look like in this case? How are they related to the other Dung semantics \cite[Section 2.3]{Dung:95}? Alternatively, one can invoke the theory of ASPIC$^+$, which states that normative rationality automatically follows if the instantiation is \textit{well-defined} with a \textit{reasonable} argument preference relation \cite[Definitions 12 and 18]{sanjay:13}. Although it is easy to see that our instantiation is well-defined if the underlying PDT is consistent, it is not obvious whether $\precsim_{SP}$ in the partial order case is reasonable. This will be the subject of future work.

Finally, we have argued that endowing PDL with argumentation semantics provides for distributed reasoning amongst agents (in particular BOID agents for which PDL has been used to generate \textit{individual} agents' goals). Such distributed reasoning in the form of dialogue can be formalised as a generalisation of argument game proof theories for Dung frameworks \cite{Sanjay:09}, whereby agents not only can submit arguments, but locutions that implicitly define arguments providing the reasons for a given claim. We will investigate this in future work.

\bibliographystyle{abbrv}
\bibliography{APY_PhD_Library}

\end{document}